\documentclass[11pt]{article}

\usepackage[margin=1in]{geometry}
\usepackage[utf8]{inputenc} 
\usepackage[T1]{fontenc}    
\usepackage{hyperref}       
\usepackage{url}            
\usepackage{booktabs}       
\usepackage{amsfonts}       
\usepackage{nicefrac}       
\usepackage{microtype}      
\usepackage{comment}
\usepackage{times}
\usepackage{amsmath,stackrel,amsthm,amssymb,algorithm2e,graphicx,xcolor}
\usepackage{mathtools}
\usepackage{bbm}

\theoremstyle{plain}
\newtheorem{theorem}{Theorem}[section]
\newtheorem{lemma}[theorem]{Lemma}
\newtheorem{claim}[theorem]{Claim}
\newtheorem{corollary}[theorem]{Corollary}

\newtheorem{proposition}[theorem]{Proposition}

\theoremstyle{definition}
\newtheorem{definition}[theorem]{Definition}

\theoremstyle{remark}

\newcommand{\func}[2] {{#1}\left(#2\right)}
\newcommand{\tuple}[2] {\left(#1 , #2 \right)}

\newcommand{\sampleSize}{n}
\newcommand{\numOfIterations}{k}
\newcommand{\numOfMetaIterations}{t}

\newcommand{\domainVal}{x}
\newcommand{\domainSet}{\mathbf{\domainVal}}
\newcommand{\domainRV}{X}
\newcommand{\domain}{\mathcal{\domainRV}}

\newcommand{\domainOfSets}{\domain^{\sampleSize}}
\newcommand{\sampleSet}{s}
\newcommand{\sampleSetSet}{\mathbf{\sampleSet}}
\newcommand{\sampleSetRV}{S}

\newcommand{\rangeVal}{r}
\newcommand{\rangeSet}{\mathbf{\rangeVal}}
\newcommand{\rangeRV}{R}
\newcommand{\range}{\mathcal{\rangeRV}}

\newcommand{\viewVal}{v}
\newcommand{\viewSet}{\mathbf{\viewVal}}
\newcommand{\viewRV}{V}
\newcommand{\view}{\mathcal{\viewRV}}

\newcommand{\mappedRangeVal}{u}
\newcommand{\mappedRangeSet}{\mathbf{\mappedRangeVal}}
\newcommand{\mappedRangeRV}{U}
\newcommand{\mappedRange}{\mathcal{\mappedRangeRV}}

\newcommand{\query}{q}
\newcommand{\querySet}{\mathbf{\query}}
\newcommand{\queryRV}{Q}
\newcommand{\queriesFamily}{\mathcal{\queryRV}}

\newcommand{\mechanism}{M}
\newcommand{\mechanismFunc}[1] {\func{\mechanism}{#1}}

\newcommand{\analyst}{A}
\newcommand{\analystsFamily} {\mathcal{\analyst}}

\newcommand{\weight}{w}

\newcommand{\postProcessing}{f}

\newcommand{\distFamily}{\mathcal{D}}
\newcommand{\dist}{D}
\newcommand{\distInd}[1] {\dist_{#1}}
\newcommand{\distUpInd}[2] {\distInd{#1}^{#2}}
\newcommand{\distDomain} {\distInd{\domain}}
\newcommand{\distDomainOfSets} {\distInd{\domainOfSets}}
\newcommand{\distRange}[1] {\distUpInd{\range}{#1}}

\newcommand{\distDep}[3] {\distUpInd{#1|#2}{#3}}
\newcommand{\distJoint}[3] {\distUpInd{\tuple{#1}{#2}}{#3}}
\newcommand{\distProd}[3] {\distUpInd{#1\otimes#2}{#3}}
\newcommand{\distFunc}[1] {\func{\dist}{#1}}
\newcommand{\distFuncDep}[2] {\func{\dist}{#1 \,|\, #2}}

\newcommand{\distP}{P}
\newcommand{\distPInd}[1] {\distP_{#1}}
\newcommand{\distPUpInd}[2] {\distPInd{#1}^{#2}}
\newcommand{\distPDomain}[1] {\distPUpInd{\domain}{#1}}
\newcommand{\distPDomainOfSets}[1] {\distPUpInd{\domainOfSets}{#1}}
\newcommand{\distPRange}[1] {\distPUpInd{\range}{#1}}

\newcommand{\distPDep}[3] {\distPUpInd{#1|#2}{#3}}
\newcommand{\distPJoint}[3] {\distPUpInd{\tuple{#1}{#2}}{#3}}

\newcommand{\distPFunc}[2] {\func{\distPUpInd{}{#1}}{#2}}
\newcommand{\distPFuncDep}[3] {\func{\distPUpInd{}{#1}}{#2 \,|\, #3}}

\newcommand{\prob}[2] {\underset{#1}{\text{Pr}}\left[#2\right]}
\newcommand{\expectation}[2] {\underset{#1}{\mathbb{E}}\left[#2\right]}

\title{A necessary and sufficient stability notion for adaptive generalization}

\author{Katrina Ligett and Moshe Shenfeld}
\begin{document}

\maketitle

\begin{abstract}
We introduce a new notion of the stability of computations, which holds under post-processing and adaptive composition. We show that the notion is both necessary and sufficient to ensure generalization in the face of adaptivity, for any computations that respond to bounded-sensitivity linear queries while providing accuracy with respect to the data sample set. The stability notion is based on quantifying the effect of observing a computation's outputs on the posterior over the data sample elements. We show a separation between this stability notion and previously studied notion and observe that all differentially private algorithms also satisfy this notion.
\end{abstract}

\section{Introduction}
A fundamental idea behind most forms of data-driven research and machine learning is the concept of \emph{generalization}--the ability to infer properties of a data distribution by working only with a sample from that distribution. 
One typical approach is to invoke a concentration bound to ensure that, for a sufficiently large sample size, the evaluation of the function on the sample set will yield a result that is close to its value on the underlying distribution, with high probability. 
Intuitively, these concentration arguments ensure that, for any given function, most sample sets are good ``representatives'' of the distribution. Invoking a union bound, such a guarantee easily extends to the evaluation of multiple functions on the same sample set.

Of course, such guarantees hold only if the functions to be evaluated were chosen independently of the sample set. In recent years, grave concern has erupted in many data-driven fields, that \emph{adaptive selection} of computations is eroding statistical validity of scientific findings~\cite{Ioannidis05,GL14}. Adaptivity is not an evil to be avoided---it constitutes a natural part of the scientific process, wherein previous findings are used to develop and refine future hypotheses. However, unchecked adaptivity can (and does, as demonstrated by, e.g., \cite{DFHPRR15b} and \cite{RZ16}) often lead one to evaluate \emph{overfitting} functions---ones that return very different values on the sample set than on the distribution.

Traditional generalization guarantees do not necessarily guard against adaptivity; while generalization ensures that the response to a query on a sample set will be \emph{close} to that of the same query on the distribution, it does not rule out the possibility that the probability to get a \emph{specific} response will be dramatically affected by the contents of the sample set. In the extreme, a generalizing computation could encode the whole sample set in the low-order bits of the output, while maintaining high accuracy with respect to the underlying distribution. Subsequent adaptive queries could then, by \emph{post-processing} the computation's output, arbitrarily overfit to the sample set.

In recent years, an exciting line of work, starting with Dwork et al.~\cite{DFHPRR15b}, has formalized this problem of adaptive data analysis and introduced new techniques to ensure guarantees of generalization in the face of an adaptively-chosen sequence of computations (what we call here \emph{adaptive generalization}).
One great insight of Dwork et al. and followup work was that techniques for ensuring the \emph{stability} of computations (some of them originally conceived as privacy notions) can be powerful tools for providing adaptive generalization.

A number of papers have considered variants of stability notions, the relationships between them, and their properties, including generalization properties. Despite much progress in this space, one issue that has remained open is the limits of stability---how much can the stability notions be relaxed, and still imply generalization? It is this question that we address in this paper. 

\subsection{Our Contribution}
We introduce a new notion of the stability of computations, which holds under post-processing (Theorem \ref{thm:postprocessing}) and adaptive composition (Theorems \ref{thm:adaptiveComposition} and \ref{thm:advancedAdaptiveComposition}), and show that the notion is both necessary (Theorem~\ref{thm:necessary}) and sufficient (Theorem~\ref{thm:highprob}) to ensure generalization in the face of adaptivity, for any computations that respond to bounded-sensitivity linear queries (see Definition \ref{def:linearqueries}) while providing accuracy with respect to the data sample set. This means (up to a small caveat)\footnote{In particular, our lower bound (Theorem \ref{thm:necessary}) requires one more query than our upper bound (Theorem \ref{thm:highprob}).} that our stability definition is equivalent to generalization, assuming sample accuracy, for bounded linear queries.
Linear queries form the basis for many learning algorithms, such as those that rely on gradients or on the estimation of the average loss of a hypothesis.

In order to formulate our stability notion, we consider a prior distribution over the database elements and the posterior distribution over those elements conditioned on the output of a computation. In some sense, harmful outputs are those that induce large  statistical distance between this prior and posterior (Definition \ref{def:stabLoss}). Our new notion of stability, \emph{Local Statistical Stability} (Definition \ref{def:LSS}), intuitively, requires a computation to have only small probability of producing such a harmful output.

In Section ~\ref{sec:relToNotions}, we directly prove that Differential Privacy, Max Information, Typical Stability and Compression Schemes all imply Local Statistical Stability, which provides an alternative method to establish their generalization properties. We also provide a few separation results between the various definitions.

\subsection{Additional Related Work}
Most countermeasures to overfitting fall into one of a few categories. A long line of work bases generalization guarantees on some form of bound on the complexity of the range of the mechanism, e.g., its VC dimension (see \cite{SSBD14} for a textbook summary of these techniques). Other examples include \emph{Bounded Description Length} \cite{DFHPRR15a}, and \emph{compression schemes} \cite{LW86} (which additionally hold under post-processing and adaptive composition~\cite{DFHPRR15a,CLNRW16}).
Another line of work focuses on the algorithmic stability of the computation \cite{BE02}, which bounds the effects on the output of changing one element in the training set.

A different category of stability notions, which focus on the effect of a small change in the sample set on the probability distribution over the range of possible outputs, has recently emerged from the notion of Differential Privacy \cite{DMNS06}. 
Work of \cite{DFHPRR15b} established that Differential Privacy, interpreted as a stability notion, ensures generalization; it is also known (see \cite{DR14}) to be robust to adaptivity and to withstand post-processing.
A number of subsequent works propose alternative stability notions that weaken the conditions of Differential Privacy in various ways while attempting to retain its desirable generalization properties. One example is \emph{Max Information} \cite{DFHPRR15a}, which shares the guarantees of Differential Privacy.
A variety of other stability notions (\cite{RRST16,RZ16,RRTWX16,BNSSSU16,FS17,EGI19}), unlike Differential Privacy and Max Information, only imply generalization in expectation.  \cite{XR17,Al17,BMNSY17} extend these guarantees to generalization in probability, under various restrictions.

\cite{CLNRW16} introduce the notion of \emph{post-hoc generalization}, which captures robustness to post-processing, but it was recently shown not to hold under composition \cite{NSSSU18}. The challenges that the internal correlation of non-product distributions present for stability have been studied in the context of \emph{Inferential Privacy} \cite{GK16} and \emph{Typical Stability} \cite{BF16}. 
\section{LS stability definition and properties}\label{sec:defns}
Let $\domain$ be an arbitrary countable \emph{domain}. Fixing some $\sampleSize \in \mathbb{N}$, let $\distDomainOfSets$ be some probability distribution defined over $\domainOfSets$.\footnote{Throughout the paper, $\domainOfSets$ can either denote the family of sequences of length $\sampleSize$ or a multiset of size $\sampleSize$; that is, the sample set $\sampleSet$ can be treated as an ordered or unordered set.}
Let $\queriesFamily, \range$ be arbitrary countable sets which we will refer to as \emph{queries} and \emph{responses}, respectively.
Let a \emph{mechanism} $\mechanism : \domainOfSets \times \queriesFamily \rightarrow \range$ be a (possibly non-deterministic) function that, given a \emph{sample set} $\sampleSet \in \domainOfSets$ and a query $\query \in \queriesFamily$, returns a response $\rangeVal \in \range$.
Intuitively, queries can be thought of as questions the mechanism is asked about the sample set, usually representing functions from $\domainOfSets$ to $\range$; the mechanism can be thought of as providing an estimate to the value of those functions, but we do not restrict the definitions, for reasons which will become apparent once we introduce the notion of adaptivity (Definition \ref{def:adapMechan}).

This setting involves two sources of randomness, the \emph{underlying distribution} $\distDomainOfSets$, and the \emph{conditional distribution} $\func{\distDep{\range}{\domainOfSets}{\query}}{\rangeVal \,|\, \sampleSet}$---that is, the probability to get $\rangeVal$ as the output of $\mechanismFunc{\sampleSet, \query}$. These in turn induce a set of distributions (formalized in Definition \ref{def:distsOfSets}): the \emph{marginal distribution} over $\range$, the \emph{joint distribution} (denoted $\distJoint{\domainOfSets}{\range}{\query}$) and \emph{product distribution} (denoted $\distProd{\domainOfSets}{\range}{\query}$) over $\domainOfSets \times \range$, and the \emph{conditional distribution} over $\domainOfSets$ given $\rangeVal \in \range$. Note that even if $\distDomainOfSets$ is a product distribution, this conditional distribution might not be a product distribution.
Although the underlying distribution $\distDomainOfSets$ is defined over $\domainOfSets$, it induces a natural probability distribution over $\domain$ as well, by sampling one of the sample elements in the set uniformly at random.\footnote{It is worth noting that in the case where $\distDomainOfSets$ is the product distribution of some distribution $\distPDomain{}$ over $\domain$, we get that the induced distribution over $\domain$ is $\distPDomain{}$.}
This in turn allows us extend our definitions to several other distributions, which form a connection between $\range$ and $\domain$ (formalized in Definition \ref{def:distsOfDomain}): the \emph{marginal distribution} over $\domain$, the \emph{joint distribution} and \emph{product distribution} over $\domain \times \range$, the \emph{conditional distribution} over $\range$ given $\domainVal \in \domain$, and the \emph{conditional distribution} over $\domain$ given $\rangeVal \in \range$. We use our distribution notation to denote both the probability that a distribution places on a subset of its range and the probability placed on a single element of the range.

\paragraph{Notational conventions} We use calligraphic letters to denote domains, lower case letters to denote elements of these domains, capital letters to denote random variables taking values in these domains, and bold letters to denote subsets of these domains. We omit subscripts and superscripts from some notation when they are clear from context.

\subsection{Local Statistical Stability}
Before observing any output from the mechanism, an outside observer knowing $\dist$ but without other information about the sample set $\sampleSet$ holds prior $\distFunc{\domainVal}$ that sampling an element of $\sampleSet$ would return a particular $\domainVal \in \domain$. Once an output $\rangeVal$ of the mechanism is observed, however, the observer's posterior becomes $\distFuncDep{\domainVal}{\rangeVal}$. The difference between these two distributions is what determines the resulting degradation in stability. This difference could be quantified using a variety of distance measures (a partial list can be found in Appendix \ref{apd:DistMeasures}); here we introduce a particular one which we use to define our stability notion.

\begin{definition}[Stability loss of a response]\label{def:stabLoss}
Given a distribution $\distDomainOfSets$, a query $\query$, and a mechanism $\mechanism : \domainOfSets \times \queriesFamily \rightarrow \range$, the \emph{stability loss} $\func{\ell_{\distDomainOfSets}^\query}{\rangeVal}$ of a response $\rangeVal \in \range$ with respect to $\distDomainOfSets$ and $\query$ is defined as the Statistical Distance (Definition \ref{def:statDist}) between the prior distribution over $\domain$ and the posterior induced by $\rangeVal$. That is,
\[
\func{\ell_{\distDomainOfSets}^{\query}}{\rangeVal} \coloneqq \sum_{\domainVal \in \func{\domainSet_{+}}{\rangeVal}} \left( \distFuncDep{\domainVal}{\rangeVal} - \distFunc{\domainVal} \right),
\]
where $\func{\domainSet_{+}}{\rangeVal} \coloneqq \left\{\domainVal \in \domain \,|\, \distFuncDep{\domainVal}{\rangeVal} > \distFunc{\domainVal} \right\}$, the set of all sample elements which have a posterior probability (given $\rangeVal$) higher then their prior.
Similarly, we define the stability loss $\func{\ell}{\rangeSet}$ of a set of responses $\rangeSet \subseteq \range$ as
\[
\func{\ell}{\rangeSet} \coloneqq \frac{\sum_{\rangeVal \in \rangeSet} \distFunc{\rangeVal} \cdot \func{\ell}{\rangeVal}}{\distFunc{\rangeSet}}.
\]
Given $0 \le \epsilon \le 1$, a response will be called \emph{$\epsilon$-unstable} with respect to $\distDomainOfSets$ and $\query$ if its loss is greater the $\epsilon$. The set of all $\epsilon$-unstable responses will be denoted $\rangeSet_{\epsilon}^{\distDomainOfSets, \query} \coloneqq \left\{ \rangeVal \in \range \,|\, \func{\ell}{\rangeVal} > \epsilon \right\}$.
\end{definition}

We now introduce our notion of stability of a mechanism.
\begin{definition}[Local Statistical Stability]\label{def:LSS}
Given $0 \le \epsilon, \delta \le 1$, a distribution $\distDomainOfSets$, and a query $\query$, a mechanism $\mechanism : \domainOfSets \times \queriesFamily \rightarrow \range$ will be called \emph{$\tuple{\epsilon}{\delta}$-Local-Statistically Stable with respect to $\distDomainOfSets$ and $\query$} (or \emph{LS Stable}, or \emph{LSS}, for short) if for any $\rangeSet \subseteq \range$, $\distFunc{\rangeSet} \cdot \left( \func{\ell}{\rangeSet} - \epsilon\right) \le \delta$.

Notice that the maximal value of the left hand side is achieved for the subset $\rangeSet_{\epsilon}$. This stability definition can be extended to apply to a family of queries and/or a family of possible distributions. When there exists a family of queries $\queriesFamily$ and a family of distributions $\distFamily$ such that a mechanism $\mechanism$ is $\tuple{\epsilon}{\delta}$-LSS for all $\distDomainOfSets \in \distFamily$ and for all $\query \in \queriesFamily$, then $\mechanism$ will be called \emph{$\tuple{\epsilon}{\delta}$-LSS for $\distFamily, \queriesFamily$}. (This stability notion somewhat resembles \emph{Semantic Privacy} as discussed by \cite{kS14}, though they use it to compare different posterior distributions.)
\end{definition}

Intuitively, this can be thought of as placing a $\delta$ bound on the probability of observing an outcome whose stability loss exceeds $\epsilon$. This claim is formalized in the next Lemma.

\begin{lemma} \label{lem:notToManyBad}
Given $0 \le  \delta \le \epsilon \le 1$, a distribution $\distDomainOfSets$, and a query $\query$, if a mechanism $\mechanism$ is $\tuple{\epsilon}{\delta}$-LSS with respect to $\distDomainOfSets, \query$, then $\distFunc{\rangeSet_{2 \epsilon}} < \frac{\delta}{\epsilon}$.
\end{lemma}

\begin{proof}
Assume by way of contradiction that $\distFunc{\rangeSet_{2 \epsilon}} \ge \frac{\delta}{\epsilon}$; then
\[
\distFunc{\rangeSet_{2 \epsilon}} \cdot \left( \func{\ell}{\rangeSet_{2 \epsilon}} - \epsilon\right) > \frac{\delta}{\epsilon} \cdot \left(2 \epsilon - \epsilon \right) = \delta.\qedhere
\]
\end{proof}

\subsection{Properties}

We now turn to prove two crucial properties of LSS: post-processing and adaptive composition.

Post-processing guarantees (known in some contexts as data processing inequalities) ensure that the stability of a computation can only be \emph{increased} by subsequent manipulations. This is a key desideratum for concepts used to ensure adaptivity-proof generalization, since otherwise an adaptive subsequent computation could potentially arbitrarily degrade the generalization guarantees.

\begin{theorem}[LSS holds under Post-Processing]\label{thm:postprocessing}
Given $0 \le \epsilon, \delta \le 1$, a distribution $\distDomainOfSets$, and a query $\query$, if a mechanism $\mechanism$ is $\tuple{\epsilon}{\delta}$-LSS with respect to $\distDomainOfSets$ and $\query$, then for any range $\mappedRange$ and any arbitrary (possibly non-deterministic) function $\postProcessing:\range \rightarrow \mappedRange$, we have that $\postProcessing \circ \mechanism : \domainOfSets \times \queriesFamily \rightarrow \mappedRange$ is also $\tuple{\epsilon}{\delta}$-LSS with respect to $\distDomainOfSets$ and $\query$. An analogous statement also holds for mechanisms that are LSS with respect to a family of queries and/or a family of distributions.
\end{theorem}

\begin{proof} 
We start by defining a function $\weight_{\mappedRange}^{\epsilon} : \range \rightarrow \left[ 0, 1 \right]$ such that $\forall \rangeVal \in \range: \func{\weight_{\mappedRange}^{\epsilon}}{\rangeVal} = \underset{\mappedRangeVal \in \mappedRangeSet_{\epsilon}}{\sum} \distFuncDep{\mappedRangeVal}{\rangeVal}$, where $\mappedRangeSet_{\epsilon}$ is the set of $\epsilon$-unstable values in $\mappedRange$ as defined in Definition \ref{def:stabLoss}, and $\distFuncDep{\mappedRangeVal}{\rangeVal} \coloneqq \prob{\mappedRangeRV \sim \func{\postProcessing}{\rangeVal}}{\mappedRangeRV = \mappedRangeVal \,|\, \rangeVal}$. Using this function we get that,
\[
\sum_{\mappedRangeVal \in \mappedRangeSet_{\epsilon}} \distFunc{\mappedRangeVal} = \sum_{\rangeVal \in \range} \func{\weight_{\mappedRange}^{\epsilon}}{\rangeVal} \cdot \distFunc{\rangeVal},
\]
and
\[
\sum_{\mappedRangeVal \in \mappedRangeSet_{\epsilon}} \distFunc{\mappedRangeVal} \cdot \func{\ell}{\mappedRangeVal} \le \sum_{\rangeVal \in \rangeSet} \func{\weight_{\mappedRange}^{\epsilon}}{\rangeVal} \cdot \distFunc{\rangeVal} \cdot \func{\ell}{\rangeVal}
\]
(detailed proof can be found in Appendix \ref{sec:postProcProof}).

Combining the two we get that
\begin{align*}
\distFunc{\mappedRangeSet_{\epsilon}} \cdot \left( \func{\ell}{\mappedRangeSet_{\epsilon}} - \epsilon\right) & = \sum_{\mappedRangeVal \in \mappedRangeSet_{\epsilon}} \distFunc{\mappedRangeVal} \left( \func{\ell}{\mappedRangeVal} - \epsilon \right)
\\ & \overset{\left( 1 \right)}{\le} \sum_{\rangeVal \in \range} \func{\weight_{\mappedRange}^{\epsilon}}{\rangeVal} \cdot \distFunc{\rangeVal} \left( \func{\ell}{\rangeVal} - \epsilon \right)
\\ & \overset{\left( 2 \right)}{\le} \sum_{\rangeVal \in \rangeSet_{\epsilon}} \overbrace{\func{\weight_{\mappedRange}^{\epsilon}}{\rangeVal}}^{\le 1} \cdot \distFunc{\rangeVal} \left( \func{\ell}{\rangeVal} - \epsilon \right)
\\ & \le \sum_{\rangeVal \in \rangeSet_{\epsilon}} \distFunc{\rangeVal} \left( \func{\ell}{\rangeVal} - \epsilon \right)
\\ & \overset{\left( 3 \right)}{\le} \delta,
\end{align*}
where (1) results from the two previous claims, (2) from the fact that we removed only negative terms and (3) from the LSS definition, which concludes the proof.
\end{proof}

In order to formally define adaptive learning and stability under adaptively chosen queries, we formalize the notion of an analyst who issues those queries.

\begin{definition}[Analyst and Adaptive Mechanism] \label{def:adapMechan}
An \emph{analyst over a family of queries $\queriesFamily$} is a (possibly non-deterministic) function $\analyst: \range^{*} \rightarrow \queriesFamily$ that receives a \emph{view}---a finite sequence of responses---and outputs a query. We denote by $\analystsFamily$ the family of all analysts, and write $\view_{\numOfIterations} \coloneqq \range^{\numOfIterations}$ and $\view \coloneqq \range^{*}$.

Illustrated below, the \emph{adaptive mechanism} $\text{Adp}_{\bar{\mechanism}}: \domainOfSets \times \analystsFamily \rightarrow \view_{\numOfIterations}$ is a particular type of mechanism, which inputs an analyst as its query and which returns a view as its range type. It is parametrized by a sequence of \emph{sub-mechanisms} $\bar{\mechanism} = \left( \mechanism_{i} \right)_{i = 1}^{\numOfIterations}$ where $\forall i \in \left[ \numOfIterations \right]$, $\mechanism_{i} : \domainOfSets \times \queriesFamily \rightarrow \range$.
Given a sample set $\sampleSet$ and an analyst $\analyst$ as input, the adaptive mechanism iterates $\numOfIterations$ times through the process where $\analyst$ sends a query to $\mechanism_{i}$ and receives its response to that query on the sample set. The adaptive mechanism returns the resulting sequence of $\numOfIterations$ responses $\viewVal_{\numOfIterations}$. Naturally, this requires $\analyst$ to match $\mechanism$ such that $\mechanism$'s range can be $\analyst$'s input, and vice versa.\footnote{If the same mechanism appears more then once in $\bar{\mechanism}$, it can also be stateful, which means it retains an internal record consisting of internal randomness, the history of sample sets and queries it has been fed, and the responses it has produced; its behaviour may be a function of this internal record.
We omit this from the notation for simplicity, but do refer to this when relevant. A stateful mechanism will be defined as LSS if it is LSS given any reachable internal record. A pedantic treatment might consider the \emph{probability} that a particular internal state could be reached, and only require LSS when accounting for these probabilities.} \footnote{If $\analyst$ is randomized, we add one more step at the beginning where $\text{Adp}_{\bar{\mechanism}}$ randomly generates some bits $c$---$\analyst$'s ``coin tosses.'' In this case, $v_{\numOfIterations} \coloneqq \left( c, \rangeVal_{1}, \ldots, \rangeVal_{i\numOfIterations} \right)$ and $\analyst$ receives the coin tosses as an input as well. This addition turns $\query_{\numOfIterations + 1}$ into a deterministic function of $\viewVal_{i}$ for any $i \in \mathbb{N}$, a fact that will be used multiple times throughout the paper. In this situation, the randomness of $\text{Adp}_{\bar{\mechanism}}$ results both from the randomness of the coin tosses and from that of the sub-mechanisms.}
\end{definition}

\begin{figure}[h]
    \centering
\fbox{
 \begin{minipage}{\linewidth - 15pt}
Adaptive Mechanism $\text{Adp}_{\bar{\mechanism}}$
\hrule

\textbf{Input:} $\sampleSet \in \domainOfSets, \, \analyst \in \analystsFamily$

\textbf{Output:} $\viewVal_{\numOfIterations} \in \view_{\numOfIterations}$

$\viewVal_{0} \gets \emptyset$ or $c$

\textbf{for} $i \in \left[ \numOfIterations \right]$ :

~~~$\query_{i} \gets \func{\analyst}{\viewVal_{i-1}}$

~~~$\rangeVal_{i} \gets \func{\mechanism_{i}}{\sampleSet, \query_{i}}$

~~~$\viewVal_{i} \gets \tuple{\viewVal_{i-1}}{ \rangeVal_{i}}$

\textbf{return} $\viewVal_{\numOfIterations}$
 \end{minipage}
}
\end{figure}

For illustration, consider a gradient descent algorithm, where at each step the algorithm requests an estimate of the gradient at a given point, and chooses the next point in which the gradient should be evaluated based on the response it receives. For us, $\mechanism$ evaluates the gradient at a given point, and $\analyst$ determines the next point to be considered. The interaction between the two of them constitutes an adaptive learning process.

\begin{definition}[$k$-LSS under adaptivity]
Given $0 \le \epsilon, \delta \le 1$, a distribution $\distDomainOfSets$, and an analyst $\analyst$, a sequence of $\numOfIterations$ mechanisms $\bar{\mechanism}$ will be called \emph{$\tuple{\epsilon}{\delta}$-local-statistically stable under $\numOfIterations$ adaptive iterations} with respect to $\distDomainOfSets$ and $\analyst$ (or $\numOfIterations$-LSS for short), if $\text{Adp}_{\bar{\mechanism}}$ is $\tuple{\epsilon}{\delta}$-LSS with respect to $\distDomainOfSets$ and  $\analyst$ (in which case we will use $\viewSet_{\epsilon}^{\numOfIterations, \analyst, \distDomainOfSets}$ to denote the set of $\epsilon$ unstable views).
This definition can be extended to a family of analysts and/or a family of possible distributions as well.
\end{definition}

Adaptive composition is a key property of a stability notion, since it restricts the degradation of stability across multiple computations. 
A key observation is that the posterior $\distFuncDep{\sampleSet}{\viewVal_{\numOfIterations}}$ is itself a distribution over $\domainOfSets$ and $\query_{\numOfIterations + 1}$ is a deterministic function of $\viewVal_{\numOfIterations}$. Therefore, as long as each sub-mechanism is LSS with respect to any posterior that could have been induced by previous adaptive interaction, one can reason about the properties of the composition.

\begin{definition}[View-induced posterior distributions] \label{def:posteriorDist}
A sequence of mechanisms $\bar{\mechanism}$, an analyst $\analyst$, and a view $\viewVal_{\numOfIterations} \in \view_{\numOfIterations}$ together induce a set of posterior distributions over $\domainOfSets$, $\domain$, and $\range$. For clarity we will denote these induced distributions by $\distPUpInd{}{\viewVal_{\numOfIterations}}$ instead of $\dist$.

As mentioned before, all the distributions we consider stem from two basic distributions; the underlying distribution $\distDomainOfSets$ and the conditional distribution $\distDep{\range}{\domainOfSets}{\query}$. The posteriors of these distributions change once we see $\viewVal_{\numOfIterations}$. $\distDomainOfSets$ is replaced by $\distPDomainOfSets{\viewVal_{\numOfIterations}} \coloneqq \func{\distDep{\domainOfSets}{\view_{\numOfIterations}}{\analyst}}{\cdot \,|\, \viewVal_{\numOfIterations}}$ (actually, the rigorous notation should have been $\distPDomainOfSets{\bar{\mechanism}, \analyst, \viewVal_{\numOfIterations}}$, but since $\bar{\mechanism}$ and $\analyst$ will be fixed throughout this analysis, we omit them for simplicity). Similarly, $\func{\distDep{\range}{\domainOfSets}{\query_{\numOfIterations + 1}}}{\rangeVal \,|\, \sampleSet}$ is replaced by 
\[
\func{\distPDep{\range}{\domainOfSets}{\viewVal_{\numOfIterations}}}{\rangeVal \,|\, \sampleSet} \coloneqq \func{\distDep{\range}{\domainOfSets}{\query_{\numOfIterations + 1}}}{\rangeVal \,|\, \sampleSet, \viewVal_{\numOfIterations}} = \prob{\rangeRV \sim \func{\mechanism_{\numOfIterations + 1}}{\sampleSet, \query_{\numOfIterations + 1}}}{\rangeRV = \rangeVal \,|\, \sampleSet,  \func{\text{Adp}_{\bar{\mechanism}, \numOfIterations}}{\sampleSet, \analyst} = \viewVal_{\numOfIterations}},
\]
where $\text{Adp}_{\bar{\mechanism}, \numOfIterations}$ denotes the first $\numOfIterations$ iterations of the adaptive mechanism, which - as mentioned previously - determine the $\numOfIterations + 1$-th query.\footnote{If $\mechanism_{\numOfIterations + 1}$ is stateful, the conditioning can result from any unknown state of $\mechanism_{\numOfIterations + 1}$ which might affect its response to $\query_{\numOfIterations + 1}$. If $\mechanism_{\numOfIterations + 1}$ has no shared state with the previous sub-mechanisms (either because it is a different mechanism or because it is stateless), then the only effect $\viewVal_{\numOfIterations}$ has on the posterior on $\range$ is by governing $\query_{\numOfIterations + 1}$ (which, as mentioned before, is a deterministic function of $\viewVal_{\numOfIterations}$ for the given $\analyst$), in which case $\func{\distPDep{\range}{\domainOfSets}{\viewVal_{\numOfIterations}}}{\rangeVal \,|\, \sampleSet} = \func{\distDep{\range}{\domainOfSets}{\query_{\numOfIterations + 1}}}{\rangeVal \,|\, \sampleSet}$ where the mechanism is $\mechanism_{\numOfIterations + 1}$.}
\end{definition}

We next establish two important properties of the distributions over $\view_{\numOfIterations + 1}$ induced by $\text{Adp}_{\bar{\mechanism}}$ and their relation to the posterior distributions. 

\begin{lemma} \label{lem:liniarityOfLoss}
Given a distribution $\distDomainOfSets$, an analyst $\analyst : \view \rightarrow \queriesFamily$, and a sequence of $\numOfIterations$ mechanisms $\bar{\mechanism}$, for any $\viewVal_{\numOfIterations + 1} \in \view_{\numOfIterations + 1}$ we denote $\viewVal_{\numOfIterations + 1} = \tuple{ \viewVal_{\numOfIterations}}{ \rangeVal_{\numOfIterations + 1}}$. In this case, using notation from Definition \ref{def:posteriorDist},

\[
\distFunc{\viewVal_{\numOfIterations + 1}} = \distFunc{\viewVal_{\numOfIterations}} \cdot \distPFunc{\viewVal_{\numOfIterations}}{\rangeVal_{\numOfIterations + 1}}
\]
and
\[
\func{\ell_{\distDomainOfSets}^{\analyst}}{\viewVal_{\numOfIterations + 1}} \le \func{\ell_{\distDomainOfSets}^{\analyst}}{\viewVal_{\numOfIterations}} + \func{\ell_{\distPDomainOfSets{\viewVal_{\numOfIterations}}}^{\query_{\numOfIterations + 1}}}{\rangeVal_{\numOfIterations + 1}}.
\]
\end{lemma}
The proof can be found in Appendix \ref{sec:adapCompProof}.

We first show that the stability loss of a view is bounded by the sum of losses of its responses with respect to the sub-mechanisms, which provides a linear bound on the degradation of the LSS parameters. Adding a bound on the expectation of the loss of the sub-mechanisms allows us to also invoke Azuma's inequality and prove a sub-linear bound.

\begin{theorem} [LSS adaptively composes linearly] \label{thm:adaptiveComposition}
Given a family of distributions $\distFamily$ over $\domainOfSets$, a family of queries $\queriesFamily$, and a sequence of $\numOfIterations$ mechanisms $\bar{\mechanism}$ where $\forall i \in \left[ \numOfIterations \right]$, $\mechanism_{i} : \domainOfSets \times \queriesFamily \rightarrow \range$, we will denote $\distFamily_{\mechanism_{0}, \queriesFamily} \coloneqq \distFamily$, and for any $i > 0$, $\distFamily_{\mechanism_{i}, \queriesFamily}$ will denote the set of all posterior distributions induced by any response of $\mechanism_{i}$ with non-zero probability with respect to $\distFamily_{\mechanism_{i - 1}, \queriesFamily}$ and $\queriesFamily$ (see Definition \ref{def:posteriorDist}).

Given a sequence $0 \le \epsilon_{1}, \delta_{1}, \ldots, \epsilon_{\numOfIterations}, \delta_{\numOfIterations} \le 1$, if for all $i$, $\mechanism_{i}$ is $\tuple{\epsilon_{i}}{\delta_{i}}$-LSS with respect to $\distFamily_{\mechanism_{i - 1}, \queriesFamily}$ and $\queriesFamily$, the sequence is $\tuple{\underset{i \in \left[ \numOfIterations \right]}{\sum} \epsilon_{i}}{\underset{i \in \left[ \numOfIterations \right]}{\sum} \delta_{i}}$-$\numOfIterations$-LSS with respect to $\distFamily$ and any analyst $\analyst$ over $\queriesFamily \times \range$.
\end{theorem}

\begin{proof}
 This theorem is a direct result of combining Lemma \ref{lem:liniarityOfLoss} with the triangle inequality over the posteriors created at any iteration, and the fact that the mechanisms are LSS over the new posterior distributions. Formally this is proven using induction on the number of adaptive iterations. The base case $\numOfIterations = 0$ is the coin tossing step, which is independent of the set and therefore has zero loss. For the induction step we start by denoting the projections of $\viewSet_{\epsilon_{\left[ \numOfIterations + 1 \right]}}^{\numOfIterations + 1}$ on $\view_{\numOfIterations}$ and $\range$ by
\[
\forall \rangeVal_{\numOfIterations + 1} \in \range , \func{\viewSet_{\numOfIterations}}{\rangeVal_{\numOfIterations + 1}} \coloneqq \left\{ \viewVal_{\numOfIterations} \in \view_{\numOfIterations} \,|\, \tuple{\viewVal_{\numOfIterations}}{\rangeVal_{\numOfIterations + 1}} \in \viewSet_{\epsilon_{\left[ \numOfIterations + 1 \right]}}^{\numOfIterations + 1} \right\}
\]
\[
\forall \viewVal_{\numOfIterations} \in \view_{\numOfIterations} , \func{\rangeSet}{\viewVal_{\numOfIterations}} \coloneqq \left\{ \rangeVal_{\numOfIterations + 1} \in \range \,|\, \tuple{\viewVal_{\numOfIterations}}{\rangeVal_{\numOfIterations + 1}} \in \viewSet_{\epsilon_{\left[ \numOfIterations + 1 \right]}}^{\numOfIterations + 1} \right\},
\]
where $\epsilon_{\left[ \numOfIterations \right]} \coloneqq \underset{i \in \left[ \numOfIterations \right]}{\sum} \epsilon_{i}$.

Using this notation and that in Definition \ref{def:posteriorDist} we get that
\begin{align*}
\distFunc{\viewSet_{\epsilon_{\left[ \numOfIterations + 1 \right]}}^{\numOfIterations + 1}} & \cdot \left( \func{\ell_{\distDomainOfSets}^{\analyst}}{\viewSet_{\epsilon_{\left[ \numOfIterations + 1 \right]}}^{\numOfIterations + 1}} - \epsilon_{\left[ \numOfIterations + 1 \right]} \right)
\\ & \le \distFunc{\viewSet_{\epsilon_{\left[ \numOfIterations \right]}}^{\numOfIterations}} \left( \func{\ell_{\distDomainOfSets}^{\analyst}}{\viewSet_{\epsilon_{\left[ \numOfIterations \right]}}^{\numOfIterations}} - \epsilon_{\left[ \numOfIterations \right]} \right) + \distPFunc{\viewVal_{\numOfIterations}}{\rangeSet_{\epsilon_{\numOfIterations + 1}}^{\query_{\numOfIterations + 1}}} \left( \func{\ell_{\distPDomainOfSets{\viewVal_{\numOfIterations}}}^{\query_{\numOfIterations + 1}}}{\rangeSet_{\epsilon_{\numOfIterations + 1}}^{\query_{\numOfIterations + 1}}} - \epsilon_{\numOfIterations + 1} \right)
\\ & \le \sum_{i \in \left[ \numOfIterations \right]} \delta_{i} + \delta_{\numOfIterations + 1}
\end{align*}
Detailed proof can be found in Appendix \ref{sec:adapCompProof}.
\end{proof}

One simple case is when $\distFamily_{\mechanism_{i - 1}, \queriesFamily} = \distFamily$, and $\mechanism_{i}$ is $\tuple{\epsilon_{i}}{\delta_{i}}$-LSS with respect to $\distFamily$ and $\queriesFamily$, for all $i$.

\begin{theorem} [LSS adaptively composes sub-linearly] \label{thm:advancedAdaptiveComposition}
Under the same conditions as Theorem \ref{thm:adaptiveComposition}, given $0 \le \alpha_{1}, \ldots, \alpha_{\numOfIterations} \le 1$, such that for all $i$ and any $\distDomainOfSets \in \distFamily_{\mechanism_{i - 1}, \queriesFamily}$, and $\query \in \queriesFamily$, $\expectation{\sampleSetRV \sim \distDomainOfSets, \rangeRV \sim \func{\mechanism_{i}}{\sampleSetRV, \query}}{\func{\ell}{\rangeRV}} \le \alpha_{i}$, then for any $0 \le \delta' \le 1$, the sequence is $\tuple{\epsilon'}{\delta' + \underset{i \in \left[ \numOfIterations \right]}{\sum} \frac{\delta_{i}}{\epsilon_{i}}}$-$\numOfIterations$-LSS with respect to $\distFamily$ and any analyst $\analyst$ over $\queriesFamily \times \range$, where $\epsilon' \coloneqq \sqrt{8 \func{\ln}{\frac{1}{\delta'}}\underset{i \in \left[ \numOfIterations \right]}{\sum} \epsilon_{i}^{2}} + \underset{i \in \left[ \numOfIterations \right]}{\sum} \alpha_{i}$. 
\end{theorem}
The theorem provides a better bound then the previous one in case $\alpha_i \ll \epsilon_i$, in which case the dominating term is the first one, which is sub-linear in $\numOfIterations$. 

\begin{proof}
The proof is based on the fact that the sum of the stability losses is a martingale with respect to $\viewVal_{\numOfIterations}$, and invoking Lemma \ref{lem:extendedAzuma} (extended Azuma inequality).

Formally, for any given $\numOfIterations > 0$, we can define $\Omega_{0} \coloneqq \domainOfSets$ and $\forall i \in \left[ \numOfIterations \right] , \Omega_{i} \coloneqq \range$.\footnote{ If the analyst $\analyst$ is non-deterministic, $\Omega_{0} \coloneqq \domainOfSets \times C$, where $C$ is the set of all possible coin tosses of the analyst, as mentioned in Definition \ref{def:adapMechan}. If the mechanisms have some internal state not expressed by the responses, $\Omega_{i}$ will be the domain of those states, as mentioned in Definition \ref{def:posteriorDist}.} We define a probability distribution over $\Omega_{0}$ as $\distDomainOfSets$, and for any $i > 0$, define a probability distribution over $\Omega_{i}$ given $\Omega_{1}, \ldots, \Omega_{i - 1}$ as $\distPUpInd{}{\viewVal_{i - 1}}$ (see Definition \ref{def:posteriorDist}). We then define a sequence of functions, $y_{0} = 0$ and $\forall i > 0$,
\[
\func{y_{i}}{\sampleSet, \rangeVal_{1}, \ldots, \rangeVal_{i}} = \sum_{j = 1}^{i} \left( \func{\ell_{\distPUpInd{}{\viewVal_{j - 1}}}}{\rangeVal_{j}} - \expectation{\rangeRV \sim \distPRange{\viewVal_{j - 1}}}{\func{\ell_{\distPUpInd{}{\viewVal_{j - 1}}}}{\rangeRV}} \right).
\]
Intuitively $y_{i}$ is the sum of the first $i$ losses, with a correction term which zeroes the expectation.

Notice that these random variables are a martingale with respect to the random process $\sampleSetRV, \rangeRV_{1}, \ldots, \rangeRV_{\numOfIterations}$ since
\[
\expectation{\rangeRV_{i + 1}}{Y_{i + 1} \,|\, \sampleSetRV, \rangeRV_{1}, \ldots, \rangeRV_{i}} = \func{Y_{i}}{\sampleSetRV, \rangeRV_{1}, \ldots, \rangeRV_{i}}
\]

From the LSS definition (Definition \ref{def:LSS}) and Lemma \ref{lem:notToManyBad}, for any $i \in \left[ \numOfIterations \right]$ we get that 
\[
\prob{\rangeRV \sim \distPRange{\viewVal_{i - 1}}}{\func{\ell_{\distPUpInd{}{\viewVal_{i}}}}{\rangeRV_{i}} > 2 \epsilon_{i}} \le \frac{\delta_{i}}{\epsilon_{i}},
\]
so with probability greater than $\frac{\delta_{i + 1}}{\epsilon_{i + 1}}$,

\[
\left| Y_{i + 1} - Y_{i} \right| = \left| \func{\ell_{\distPUpInd{}{\viewRV_{i}}}}{\rangeRV_{i + 1}} - \expectation{\rangeRV \sim \distPUpInd{}{\viewRV_{i}}}{\func{\ell_{\distPUpInd{}{\viewRV_{i}}}}{\rangeRV}} \right| \le \func{\ell_{\distPUpInd{}{\viewRV_{i}}}}{\rangeRV_{i + 1}} \le 2 \epsilon_{i+1}.
\]

Using this fact we can invoke Lemma \ref{lem:extendedAzuma} and get that for any $0 \le \delta' \le 1$,
\[
\underset{\viewRV \sim \distInd{\view_{\numOfIterations}}}{\text{Pr}} \left[ \func{\ell_{\distDomainOfSets}}{\viewRV} > \epsilon' \right] \le \prob{\viewRV \sim \distInd{\view_{\numOfIterations}}}{Y_{\numOfIterations} - \overbrace{Y_{0}}^{= 0} > \sqrt{8 \func{\ln}{\frac{1}{\delta'}} \sum_{i = 1}^{\numOfIterations} \epsilon_{i}^{2}}} \le \delta' + \sum_{i = 1}^{\numOfIterations} \frac{\delta_{i}}{\epsilon_{i}}
\]

Detailed proof can be found in Appendix \ref{sec:adapCompProof}.
\end{proof}

\section{LSS is Necessary and Sufficient for Generalization}\label{Generalization}
Up until this point, queries and responses have been fairly abstract concepts. In order to discuss generalization and accuracy, we must make them concrete. As a result, in this section, we often consider queries in the family of functions $\query : \domainOfSets \rightarrow \range$, and consider responses which have some metric defined over them. 
We show our results for a fairly general class of functions known as bounded linear queries.\footnote{For simplicity, throughout the following section we choose $\range = \mathbb{R}$, but all results extend to any metric space, in particular $\mathbb{R}^d$.}

\begin{definition}[Linear queries]\label{def:linearqueries}
A function $\query : \domainOfSets \rightarrow \mathbb{R}$ will be called a \emph{linear query}, if it is defined by a function $\query_{1} : \domain \rightarrow \mathbb{R}$ such that $\func{\query} {\sampleSet} \coloneqq \frac{1}{\sampleSize} \stackrel[i=1]{\sampleSize}{\sum}\func{\query_{1}}{s_{i}}$ (for simplicity we will denote $\query_{1}$ simply as $\query$ throughout the paper). If $\query: \domain \rightarrow \left[-\Delta, \Delta \right]$ it will be called a \emph{$\Delta$-bounded linear query}. The set of $\Delta$-bounded linear queries will be denoted $\queriesFamily_{\Delta}$.
\end{definition}

In this context, there is a ``correct'' answer the mechanism can produce for a given query, defined as the value of the function on the sample set or distribution, and its distance from the response provided by the mechanism can be thought of as the mechanism's error.
\begin{definition}[Sample accuracy, distribution accuracy]
Given $0 \le \epsilon$, $0 \le \delta \le 1$, a distribution $\distDomainOfSets$, and a query $\query$, a mechanism $\mechanism : \domainOfSets \times \queriesFamily \rightarrow \mathbb{R}$ will be called \emph{$\tuple{\epsilon}{\delta}$-Sample Accurate with respect to $\distDomainOfSets$ and $\query$}, if 
\[
\prob{\sampleSetRV \sim \distDomainOfSets, \rangeRV \sim \mechanismFunc{\sampleSetRV, \query}}{\left| \rangeRV - \func{\query}{\sampleSetRV} \right| > \epsilon} \le \delta.
\]
Such a mechanism will be called \emph{$\tuple{\epsilon}{\delta}$-Distribution Accurate with respect to $\distDomainOfSets$ and $\query$} if 
\[
\prob{\sampleSetRV \sim \distDomainOfSets, \rangeRV \sim \mechanismFunc{\sampleSetRV, \query}}{\left| \rangeRV - \func{\query}{\distDomainOfSets} \right| > \epsilon} \le \delta,
\]
where $\func{\query}{\distDomainOfSets} \coloneqq \expectation{\sampleSetRV \sim \distDomainOfSets}{\func{\query}{\sampleSetRV}}$.
When there exists a family of distributions $\distFamily$ and a family of queries $\queriesFamily$ such that a mechanism $\mechanism$ is $\tuple{\epsilon}{\delta}$-Sample (Distribution) Accurate for all $\dist \in \distFamily$ and for all $\query \in \queriesFamily$, then $\mechanism$ will be called \emph{$\tuple{\epsilon}{\delta}$-Sample (Distribution) Accurate with respect to $\distFamily$ and $\queriesFamily$}.

A sequence of $\numOfIterations$ mechanisms $\bar{\mechanism}$ where $\forall i \in \left[ \numOfIterations \right] : \mechanism_{i} : \domainOfSets \times \queriesFamily \rightarrow \mathbb{R}$ which respond to a sequence of $\numOfIterations$ (potentially adaptively chosen) queries $\query_1, \ldots \query_{\numOfIterations}$ will be called \emph{$\tuple{\epsilon}{\delta}$-$\numOfIterations$-Sample Accurate with respect to $\distDomainOfSets$ and $q_1, \ldots q_{\numOfIterations}$} if 
\[
\prob{\sampleSetRV \sim \distDomainOfSets, \rangeRV_{i} \sim \func{\mechanism_{i}}{\sampleSetRV, \query_{i}}}{\underset{i \in \numOfIterations}{\max} \left| \rangeRV_{i} - \func{\query_i}{\sampleSetRV} \right| > \epsilon} \le \delta,
\]
and \emph{$\tuple{\epsilon}{\delta}$-$\numOfIterations$-Distribution Accurate with respect to $\distDomainOfSets$ and $q_1, \ldots q_{\numOfIterations}$} if
\[
\prob{\sampleSetRV \sim \distDomainOfSets, \rangeRV_{i} \sim \func{\mechanism_{i}}{\sampleSetRV, \query_{i}}}{\underset{i \in \numOfIterations}{\max} \left| \rangeRV_{i} - \func{\query_i}{\distDomainOfSets} \right| > \epsilon} \le \delta.
\]
When considering an adaptive process, accuracy is defined with respect to the analyst, and the probabilities are taken also over the choice of the coin tosses by the adaptive mechanism.\footnote{If the adaptive mechanism invokes a stateful sub-mechanism multiple times, we specify that the mechanism is Sample (Distribution) Accurate if it is Sample (Distribution) Accurate given any reachable internal record. Again, a somewhat more involved treatment might consider the \emph{probability} that a particular internal state of the mechanism could be reached.}
\end{definition}

We denote by $\mathbb{V}$ the set of views consisting of responses in $\mathbb{R}$.

\subsection{LSS Implies Generalization}
As a step toward showing that LS Stability implies a high probability generalization, we first show a generalization of expectation result.
We do so, as a tool, specifically for a mechanism that returns a query as its output. Intuitively, this allows us to wrap an entire adaptive process into a single mechanism. Analyzing the potential of the mechanism to generate an overfitting query is a natural way to learn about the generalization capabilities of the mechanism.

\begin{theorem}[Generalization of expectation]\label{thm:genOfExpectation}
Given $0 \le \epsilon, \delta \le 1$, a distribution $\distDomainOfSets$, a query $\query$, and a mechanism $\mechanism : \domainOfSets \times \queriesFamily \rightarrow \queriesFamily_{\Delta}$, if $\distFunc{\querySet_{\epsilon}} < \delta$, then
\[
\left| \expectation{\sampleSetRV \sim \distDomainOfSets, \queryRV' \sim \mechanismFunc{\sampleSetRV, \query}}{\func{\queryRV'}{\distDomainOfSets} - \func{\queryRV'}{\sampleSetRV}} \right| < 2 \Delta \left(\epsilon + \delta \right).
\]
\end{theorem}

\begin{proof}
First notice that,
\[
\func{\query}{\sampleSet} = \frac{1}{\sampleSize} \sum_{i=1}^{\sampleSize} \func{\query}{\sampleSet_{i}} = \sum_{\domainVal \in \domain} \distFuncDep{\domainVal}{\sampleSet} \cdot \func{\query}{\domainVal}
\]
where $\sampleSet_{1}, \ldots \sampleSet_{\sampleSize}$ denotes the elements of the sample set $\sampleSet$.
Using this identity we separately analyze the expected value of the returned query with respect to the distribution, and with respect to the sample set (detailed proof can be found in Appendix \ref{sec:genProofs}).

\[
\expectation{\sampleSetRV \sim \distDomainOfSets, \queryRV' \sim \mechanismFunc{\sampleSetRV, \query}}{\func{\queryRV'}{\distDomainOfSets}} = \sum_{\query' \in \queriesFamily_{\Delta}} \distFunc{\query'} \sum_{\domainVal \in \domain} \distFunc{\domainVal} \cdot \func{\query'}{\domainVal}
\]

\[
\expectation{\sampleSetRV \sim \distDomainOfSets, \queryRV' \sim \mechanismFunc{\sampleSetRV, \query}}{\func{\queryRV'}{\sampleSetRV}} = \sum_{\query' \in \queriesFamily_{\Delta}} \distFunc{\query'} \sum_{\domainVal \in \domain} \distFuncDep{\domainVal}{\query'} \cdot \func{\query'}{\domainVal}
\]

Now we can calculate the difference:
\begin{align*}
\left| \expectation{\sampleSetRV \sim \distDomainOfSets, \queryRV' \sim \mechanismFunc{\sampleSetRV, \query}}{\func{\queryRV'}{\distDomainOfSets} - \func{\queryRV'}{\sampleSetRV}} \right| & = \left| \sum_{\query' \in \queriesFamily_{\Delta}} \distFunc{\query'} \sum_{\domainVal \in \domain} \left( \distFunc{\domainVal} - \distFuncDep{\domainVal}{\query'} \right) \cdot \func{\query'}{\domainVal} \right|
\\ & \overset{\left( 1 \right)}{\le} \sum_{\query' \in \queriesFamily_{\Delta}} \distFunc{\query} \overbrace{\sum_{\domainVal \in \domain} \left| \distFunc{\domainVal} - \distFuncDep{\domainVal}{\query'} \right|}^{= 2 \func{\ell}{\query'}} \cdot \Delta
\\ & = 2 \Delta \cdot \left( \overbrace{\sum_{\query' \notin \querySet_{\epsilon}} \distFunc{\query'}}^{\le 1} \cdot \overbrace{\func{\ell}{\query'}}^{\le \epsilon} + \overbrace{\sum_{\query' \in \querySet_{\epsilon}} \distFunc{\query'}}^{< \delta} \cdot \overbrace{\func{\ell}{\query'}}^{\le 1} \right)
\\ & \overset{\left( 2 \right)}{<} 2 \Delta \left(\epsilon + \delta \right),
\end{align*}
where (1) results from the definition of $\queriesFamily_{\Delta}$ and the triangle inequality, and (2) from the condition that $\distFunc{\querySet_{\epsilon}} < \delta$.
\end{proof}

\begin{corollary}
Given $0 \le \epsilon, \delta \le 1$, a distribution $\distDomainOfSets$, and a query $\query$, if a mechanism $\mechanism : \domainOfSets \times \queriesFamily \rightarrow \queriesFamily_{\Delta}$ is $\tuple{\epsilon}{\delta}$-LSS with respect to $\distDomainOfSets, \query$, then
\[
\left| \expectation{\sampleSetRV \sim \distDomainOfSets, \queryRV' \sim \mechanismFunc{\sampleSetRV, \query}}{\func{\queryRV'}{\distDomainOfSets} - \func{\queryRV'}{\sampleSetRV}} \right| < 2 \Delta \left(2 \epsilon + \frac{\delta}{\epsilon} \right).
\]
\end{corollary}

\begin{proof}
This is a direct result of combining Theorem \ref{thm:genOfExpectation} with Lemma \ref{lem:notToManyBad}.
\end{proof}

We proceed to lift this guarantee from expectation to high probability, using a thought experiment known as the Monitor Mechanism, which was introduced by \cite{BNSSSU16}. Intuitively, it runs a large number of independent copies of an underlying mechanism, and exposes the results of the least-distribution-accurate copy as its output.

\begin{definition}[The Monitor Mechanism]
The \emph{Monitor Mechanism} is a function $\text{Mon}_{\bar{\mechanism}}: \left(\domainOfSets \right)^{\numOfMetaIterations} \times \analystsFamily \rightarrow \queriesFamily \times \mathbb{R} \times \left[\numOfMetaIterations \right]$ which is parametrized by a sequence of $\numOfIterations$ mechanisms $\bar{\mechanism}$ where $\forall i \in \left[ \numOfIterations \right]$, $\mechanism_{i} : \domainOfSets \times \queriesFamily \rightarrow \mathbb{R}$. Given a series of sample sets $\bar{\sampleSet} \in \left(\domainOfSets \right)^{\numOfMetaIterations}$ and analyst $\analyst$ as input, it runs the adaptive mechanism between $\bar{\mechanism}$ and $\analyst$ for $\numOfMetaIterations$ independent times (which in particular means neither of them share state across those iterations) and outputs a query $\query \in \queriesFamily$, response $\rangeVal \in \mathbb{R}$ and index $i \in \numOfMetaIterations$, based on the following process: 
\begin{center}
\fbox{
 \begin{minipage}{\linewidth - 15pt}
Monitor Mechanism $\text{Mon}_{\bar{\mechanism}}$
\hrule

\textbf{Input:} $\bar{\sampleSet} \in \left(\domainOfSets \right)^{\numOfMetaIterations}, \analyst \in \analystsFamily$

\textbf{Output:} $\query \in \queriesFamily, \rangeVal \in \mathbb{R}, i \in \numOfMetaIterations$

\textbf{for} $i = 1,...,\numOfMetaIterations$ :

~~~$\viewVal^{i} \gets \func{\text{Adp}_{\bar{\mechanism}}}{\sampleSet_{i}, \analyst}$

~~~$\tuple{\tilde{\query}^{i}}{\tilde{\rangeVal}^{i}} \gets \underset{\tuple{\query}{\rangeVal} \in \viewVal^{i}}{\arg\max} \left|\func{\query}{\distDomainOfSets} - \rangeVal \right|$\footnote{We slightly abuse notation since $\query$ is not part of $\viewVal^{i}$, but since it can be recovered from it, this term is well defined.}

~~~\textbf{if} $\func{\tilde{\query}^{i}}{\distDomainOfSets} \ge \tilde{\rangeVal}^{i}$:\footnote{
The addition of this condition ensures that $\func{\query}{\distDomainOfSets} \ge \rangeVal$ for the output of the mechanism, a fact that will be used later in the proof of Claim \ref{clm:Mondisterror}.}

~~~~~~$\query^{i} \gets \tilde{\query}^{i}$

~~~~~~$\rangeVal^{i} \gets \tilde{\rangeVal}^{i}$

~~~\textbf{else}:

~~~~~~$\query^{i} \gets -\tilde{\query}^{i}$

~~~~~~$\rangeVal^{i} \gets -\tilde{\rangeVal}^{i}$

$i^{*} \gets \underset{i \in \left[\numOfMetaIterations \right]}{\arg\max} \left(\func{\query^{i}}{\distDomainOfSets} - \rangeVal^{i} \right)$

\textbf{return} $\left(\query^{i^{*}}, \rangeVal^{i^{*}}, i^{*} \right)$
 \end{minipage}
}
\end{center}
\end{definition}

Notice that the monitor mechanism makes use of the ability to evaluate queries according to the true underlying distribution.\footnote{Of course, no realistic mechanism would have such an ability; the monitor mechanism is simply a thought experiment used as a proof technique.} 

We begin by proving a few properties of the monitor mechanism. In this claim and the following ones, the probabilities and expectations are taken over the randomness of the choice of $\bar{\sampleSet} \in \left(\domainOfSets \right)^{\numOfMetaIterations}$ (which is assumed to be drawn iid from $\distDomainOfSets$) and the internal probability of $\text{Adp}_{\bar{\mechanism}}$. Proofs can be found in Appendix \ref{sec:genProofs}.

\begin{claim}\label{clm:monLSS}
Given $0 \le \epsilon, \delta \le 1$, $\numOfMetaIterations \in \mathbb{N}$, a distribution $\distDomainOfSets$, and an analyst $\analyst : \mathbb{V} \rightarrow \queriesFamily_{\Delta}$, if a sequence of $\numOfIterations$ mechanisms $\bar{\mechanism}$ where $\forall i \in \left[ \numOfIterations \right]$, $\mechanism_{i} : \domainOfSets \times \queriesFamily_{\Delta} \rightarrow \mathbb{R}$ is $\tuple{\epsilon}{\delta}$-$\numOfIterations$-LSS with respect to $\distDomainOfSets, \analyst$, then
\[
\left| \expectation{\bar{\sampleSetRV} \sim \distDomainOfSets^{\numOfMetaIterations}, \left( \queryRV, \rangeRV, I \right) \sim \func{\text{Mon}_{\bar{\mechanism}}}{\bar{\sampleSetRV}, \analyst}}{\func{\query}{\distDomainOfSets} - \func{\queryRV}{\sampleSetRV_{I}}} \right| < 2 \Delta \left(2 \epsilon + \frac{\numOfMetaIterations \delta}{\epsilon} \right).
\]
\end{claim}

\begin{claim}\label{clm:Monsampleerror}
Given $0 \le \epsilon, \delta \le 1$, $\numOfMetaIterations \in \mathbb{N}$, a distribution $\distDomainOfSets$, and an analyst $\analyst : \mathbb{V} \rightarrow \queriesFamily_{\Delta}$, if a sequence of $\numOfIterations$ mechanisms $\bar{\mechanism}$ where $\forall i \in \left[ \numOfIterations \right] , \mechanism_{i} : \domainOfSets \times \queriesFamily_{\Delta} \rightarrow \mathbb{R}$ is $\tuple{\epsilon}{\delta}$-$\numOfIterations$-Sample Accurate with respect to $\distDomainOfSets, \analyst$, then
\[
\expectation{\bar{\sampleSetRV} \sim \distDomainOfSets^{\numOfMetaIterations}, \left( \queryRV, \rangeRV, I \right) \sim \func{\text{Mon}_{\bar{\mechanism}}}{\bar{\sampleSetRV}, \analyst}}{\func{\queryRV}{\sampleSetRV_{I}} - \rangeRV} \le \epsilon + 2 \numOfMetaIterations \delta \Delta.
\]
\end{claim}

\begin{claim}\label{clm:Mondisterror}
Given $0 \le \epsilon, \delta \le 1$, $\numOfMetaIterations \in \mathbb{N}$, a distribution $\distDomainOfSets$, and an analyst $\analyst : \mathbb{V} \rightarrow \queriesFamily_{\Delta}$, if a a sequence of $\numOfIterations$ mechanisms $\bar{\mechanism}$ where $\forall i \in \left[ \numOfIterations \right] , \mechanism_{i} : \domainOfSets \times \queriesFamily_{\Delta} \rightarrow \mathbb{R}$ is \textbf{not} $\tuple{\epsilon}{\delta}$-$\numOfIterations$-Distribution Accurate with respect to $\distDomainOfSets, \analyst$, then
\[
\expectation{\bar{\sampleSetRV} \sim \distDomainOfSets^{\numOfMetaIterations}, \left( \queryRV, \rangeRV, I \right) \sim \func{\text{Mon}_{\bar{\mechanism}}}{\bar{\sampleSetRV}, \analyst}}{\func{\queryRV}{\distDomainOfSets} - \rangeRV} > \epsilon \left(1 - \left(1 -\delta \right)^{\numOfMetaIterations} \right).
\]
\end{claim}

Finally, we combine these claims to show that LSS implies generalization with high probability.

\begin{theorem}[LSS implies generalization with high probability] \label{thm:highprob}
Given $0 \le \epsilon \le \Delta$, $0 \le \delta \le 1$, a distribution $\distDomainOfSets$, and an analyst $\analyst : \mathbb{V} \rightarrow \queriesFamily_{\Delta}$, if a sequence of $\numOfIterations$ mechanisms $\bar{\mechanism}$ where $\forall i \in \left[ \numOfIterations \right] , \mechanism_{i} : \domainOfSets \times \queriesFamily_{\Delta} \rightarrow \mathbb{R}$ is both $\tuple{\frac{\epsilon}{8 \Delta}}{\frac{\epsilon^{2} \delta}{4800 \Delta^{2}}}$-$\numOfIterations$-LSS and $\left(\frac{\epsilon}{8},\frac{\epsilon \delta}{600 \Delta}\right)$-$\numOfIterations$-Sample Accurate with respect to $\distDomainOfSets$ and $\analyst$, then it is $\tuple{\epsilon}{\delta}$-$\numOfIterations$-Distribution Accurate with respect to $\distDomainOfSets$ and $\analyst$.
\end{theorem}

\begin{proof}
We will prove a slightly more general claim. For every $0 < a, b, c, d$ such that $4 a + 2 b + c + 2 d < 1 - e^{-1} $, say $\mechanism$ is both $\left(a \frac{\epsilon}{\Delta}, a b \frac{\epsilon^{2} \delta}{\Delta^{2}} \right)$-$\numOfIterations$-LSS and $\left(c \epsilon, d \frac{\epsilon \delta}{\Delta} \right)$-$\numOfIterations$-Sample Accurate and assume $\mechanism$ is not $\tuple{\epsilon}{\delta}$-$\numOfIterations$-Distribution Accurate.

Setting $\numOfMetaIterations = \left \lfloor \frac{1}{\delta} \right \rfloor $, we see
\begin{align*}
\left| \expectation{\bar{\sampleSetRV} \sim \distDomainOfSets^{\numOfMetaIterations}, \left( \queryRV, \rangeRV, I \right) \sim \func{\text{Mon}_{\bar{\mechanism}}}{\bar{\sampleSetRV}, \analyst}}{\func{\queryRV}{\distDomainOfSets} - \func{\queryRV}{\sampleSetRV_{I}}} \right| & \overset{\left( 1 \right)}{<} 2 \Delta \left( 2 \frac{a \epsilon}{\Delta} + \frac{\numOfMetaIterations \Delta}{a \epsilon} \cdot \frac{a b \epsilon^{2} \delta}{\Delta^{2}} \right)
\\ & \overset{\left( 2 \right)}{\le} \left( 4 a + 2 b \right) \epsilon,
\end{align*}
where (1) is a direct result of Claim \ref{clm:monLSS} and (2) uses the definition of $\numOfMetaIterations$.

But on the other hand,
\begin{align*}
\left| \expectation{\bar{\sampleSetRV} \sim \distDomainOfSets^{\numOfMetaIterations}, \left( \queryRV, \rangeRV, I \right) \sim \func{\text{Mon}_{\bar{\mechanism}}}{\bar{\sampleSetRV}, \analyst}}{\func{\queryRV}{\distDomainOfSets} - \func{\queryRV}{\sampleSetRV_{I}}} \right| & \overset{\left( 1 \right)}{\ge} \left| \expectation{\bar{\sampleSetRV} \sim \distDomainOfSets^{\numOfMetaIterations}, \left( \queryRV, \rangeRV, I \right) \sim \func{\text{Mon}_{\bar{\mechanism}}}{\bar{\sampleSetRV}, \analyst}}{\func{\queryRV}{\distDomainOfSets} - \rangeRV} \right| 
\\ &~~~ - \left| \expectation{\bar{\sampleSetRV} \sim \distDomainOfSets^{\numOfMetaIterations}, \left( \queryRV, \rangeRV, I \right) \sim \func{\text{Mon}_{\bar{\mechanism}}}{\bar{\sampleSetRV}, \analyst}}{\func{\queryRV}{\sampleSetRV_{I}} - \rangeRV} \right|
\\ & \overset{\left( 2 \right)}{>} \epsilon \left(1 - \left(1 - \delta \right)^{\numOfMetaIterations} \right) - \left(c \epsilon + 2 \numOfMetaIterations \cdot \frac{d \epsilon \delta}{\Delta} \Delta \right)
\\ & \overset{\left( 3 \right)}{>} \epsilon \left(1 - e^{-\delta \left\lfloor \frac{1}{\delta} \right\rfloor} \right) - \left(c + 2 d \right) \epsilon
\\ & \overset{\left( 4 \right)}{\ge} \epsilon \left(1 - e^{-1} \right) - \left(c + 2 d \right) \epsilon
\\ & \overset{\left( 5 \right)}{>} \left( 4 a + 2 b \right) \epsilon,
\end{align*}
where (1) is the triangle inequality, (2) uses Claims \ref{clm:Monsampleerror} and \ref{clm:Mondisterror}, (3) the definition of $\numOfMetaIterations$, (4) the inequality $1 - \delta \le e^{-\delta}$, and (5) the definition of $a, b, c, d$.
Since combining all of the above leads to a contradiction, we know that $\bar{\mechanism}$ must be $\tuple{\epsilon}{\delta}$-Distribution Accurate, which concludes the proof. The theorem was stated choosing $a = c = \frac{1}{8}, b = d = \frac{1}{600}$.
\end{proof}

\subsection{LSS is Necessary for Generalization}

We next show that a mechanism that is not LSS cannot be both Sample Accurate and Distribution Accurate.
In order to prove this theorem, we show how to construct a ``bad'' query.
\begin{definition}[Loss assessment query]\label{Loss assessment query} 
Given a query $\query$ and a response $\rangeVal$, we will define the \emph{Loss assessment query} $\tilde{\query}_{\rangeVal}$ as
\[
\func{\tilde{\query}_{\rangeVal}}{\domainVal} = 
\begin{cases}
\Delta & \distFunc{\domainVal} > \distFuncDep{\domainVal}{\rangeVal}
\\-\Delta & \distFunc{\domainVal} \le \distFuncDep{\domainVal}{\rangeVal}
\end{cases}.
\]
Intuitively, this function maximizes the difference between $\expectation{\domainRV \sim \distDomain}{\func{\tilde{\query}_{\rangeVal}}{\domainRV}}$ and $\expectation{\domainRV \sim \distDep{\domain}{\range}{\query}}{\func{\tilde{\query}_{\rangeVal}}{\domainRV} \,|\, \rangeVal}$, and as a result, the potential to overfit.\footnote{The fact that we are able to define such a query is a result of the way the distance measure of LSS treats the $\domainVal$'s and the fact that it is defined over $\domain$ and not $\domainOfSets$.}
\end{definition}

This function is used to lower bound the effect of the stability loss on the expected overfitting.

\begin{lemma} [Loss assessment query overfits in expectation] \label{lem:lossOverfits}
Given $0 \le \epsilon, \delta \le 1$, a distribution $\distDomainOfSets$, a query $\query$, and a mechanism $\mechanism$, if $\func{\dist}{\rangeSet_{\epsilon}} > \delta$, then there is a function $\postProcessing : \range \rightarrow \queriesFamily_{\Delta}$ such that,
\[
\left| \expectation{\sampleSetRV \sim \distDomainOfSets, \queryRV' \sim \postProcessing \circ \func{\mechanism}{\sampleSetRV, \query}}{\func{\queryRV'}{\distDomainOfSets} - \func{\queryRV'}{\sampleSetRV}} \right|
> 2 \epsilon \Delta \delta.
\]
\end{lemma}

\begin{proof}
Choosing $\func{\postProcessing}{\rangeVal} = \query_{\rangeVal}$ we get that,
\begin{align*}
\left| \expectation{\sampleSetRV \sim \distDomainOfSets, \queryRV' \sim \postProcessing \circ \func{\mechanism}{\sampleSetRV, \query}}{\func{\queryRV'}{\distDomainOfSets} - \func{\queryRV'}{\sampleSetRV}} \right| & \overset{\left( 1 \right)}{=} \left| \sum_{\query' \in \queriesFamily_{\Delta}} \distFunc{\query'} \cdot \sum_{\domainVal \in \domain} \left( \distFunc{\domainVal} - \distFuncDep{\domainVal}{\query'} \right) \cdot \func{\query'}{\domainVal} \right|
\\ & = \left| \sum_{\rangeVal \in \range} \distFunc{\rangeVal} \cdot \sum_{\domainVal \in \domain} \left( \distFunc{\domainVal} - \distFuncDep{\domainVal}{\rangeVal} \right) \cdot \func{\tilde{\query}_{\rangeVal}}{\domainVal} \right|
\\ & \overset{\left( 2 \right)}{\ge} \overbrace{\sum_{\rangeVal \in \rangeSet_{\epsilon}} \distFunc{\rangeVal}}^{ \ge \delta} \cdot \overbrace{\sum_{\domainVal \in \domain} \left| \distFunc{\domainVal} - \distFuncDep{\domainVal}{\rangeVal} \right|}^{= 2 \func{\ell}{\rangeVal} > 2 \epsilon} \cdot \Delta
\\ & \overset{\left( 3 \right)}{>} 2 \epsilon \Delta \delta
\end{align*}
where (1) is justified in the proof of Theorem \ref{thm:genOfExpectation}, (2) results from the definition of the loss assessment query, and (3) from the definition of $\rangeSet_{\epsilon}$. \end{proof}

We use this method for constructing an overfitting query for non-LSS mechanism, in a slight modification of the Monitor Mechanism.

\begin{definition}[The Second Monitor Mechanism]
The \emph{Second Monitor Mechanism} is a function $\text{Mon2}_{\bar{\mechanism}}: \left(\domainOfSets \right)^{\numOfMetaIterations} \times \analystsFamily \rightarrow \queriesFamily \times \mathbb{R} \times \left[\numOfMetaIterations \right]$ which is parametrized by a sequence of $\numOfIterations$ mechanisms $\bar{\mechanism}$ where $\forall i \in \left[ \numOfIterations \right] , \mechanism_{i} : \domainOfSets \times \queriesFamily \rightarrow \mathbb{R}$. Given a series of sample sets $\bar{\sampleSet} \in \left(\domainOfSets \right)^{\numOfMetaIterations}$ and analyst $\analyst$ as input, it runs the adaptive mechanism between $\bar{\mechanism}$ and $\analyst$ for $\numOfMetaIterations$ independent times and outputs a query $\query \in \queriesFamily$, response $\rangeVal \in \mathbb{R}$ and index $i \in \numOfMetaIterations$, based on the following process:

\begin{center}
\fbox{
 \begin{minipage}{\linewidth - 15pt}
Second Monitor Mechanism $\text{Mon2}_{\bar{\mechanism}}$
\hrule

\textbf{Input:} $\bar{\sampleSet} \in \left(\domainOfSets \right)^{\numOfMetaIterations}, \analyst \in \analystsFamily$

\textbf{Output:} $\query \in \queriesFamily, \rangeVal \in \mathbb{R}, i \in \numOfMetaIterations$

\textbf{for} $i = 1,...,\numOfMetaIterations$ :

~~~$\viewVal^{i} \gets \func{\text{Adp}_{\bar{\mechanism}}}{\sampleSet_{i}, \analyst}$

~~~$\query^{i} \gets \tilde{\query}_{\viewVal^{i}}$

~~~$\rangeVal^{i} \gets \mechanismFunc{\sampleSet, \query^{i}}$

$i^{*} \gets \underset{i \in \left[\numOfMetaIterations \right]}{\arg\max} \left(\func{\ell_{\distDomainOfSets}^{\analyst}}{\viewVal^{i}} \right)$

\textbf{return} $\left(\query^{i^{*}}, \rangeVal^{i^{*}}, i^{*} \right)$
 \end{minipage}
}
\end{center}
\end{definition}

Using this mechanism, we show that LSS is necessary in order for a mechanism to be both sample accurate and distribution accurate.

\begin{theorem}[Necessity of LSS for Generalization]~\label{thm:necessary}
Given $0 \le \epsilon \le \Delta$, $0 \le \delta \le 1$, a distribution $\distDomainOfSets$, and an analyst $\analyst : \mathbb{V} \rightarrow \queriesFamily_{\Delta}$, if a sequence of $\numOfIterations$ mechanisms $\bar{\mechanism}$ where $\forall i \in \left[ \numOfIterations \right] , \mechanism_{i} : \domainOfSets \times \queriesFamily_{\Delta} \rightarrow \mathbb{R}$ is not $\tuple{\frac{\epsilon}{\Delta}}{\delta}$-$\numOfIterations$-LSS, then it cannot be both $\tuple{\frac{\epsilon}{5}}{\frac{\epsilon \delta}{5 \Delta}}$ $\left(\numOfIterations + 1\right)$-Distribution Accurate and $\tuple{\frac{\epsilon}{5}}{\frac{\epsilon \delta}{5 \Delta}}$ $\left(\numOfIterations + 1\right)$-Sample Accurate.
\end{theorem}

\begin{proof}
Again we will prove a slightly more general claim. For every $0 < a, b, c, d$ such that $a + 2 b + c + 2 d < 2 \left( 1 - e^{-1} \right) $, say $\mechanism$ is both $\tuple{a \epsilon}{b \frac{\epsilon \delta}{\Delta}}$ $\left(\numOfIterations + 1 \right)$-Sample Accurate and $\tuple{c \epsilon}{d \frac{\epsilon \delta}{\Delta}}$ $\left(\numOfIterations + 1 \right)$-Distribution Accurate and assume $\mechanism$ is not $\tuple{\frac{\epsilon}{\Delta}}{\delta}$-$\numOfIterations$-LSS.

First notice that if $\bar{\mechanism}$ is not $\tuple{\frac{\epsilon}{\Delta}}{\delta}$-$\numOfIterations$-LSS with respect to $\distDomainOfSets, \analyst$, then in particular
$\distFunc{\viewSet_{\left( \frac{\epsilon}{\Delta} \right)}^{\numOfIterations}} \ge \delta$. Since the $\numOfMetaIterations$ rounds of the second monitor mechanism are independent and $i^{*}$ is the index of the round with the maximal stability loss of the calculated query, we get that
\[
\prob{\bar{\sampleSetRV} \sim \distDomainOfSets^{\numOfMetaIterations}, \left( \queryRV, \rangeRV, I \right) \sim \func{\text{Mon2}_{\bar{\mechanism}}}{\bar{\sampleSetRV}, \analyst}}{\viewVal^{I} \in \viewSet_{\left( \frac{\epsilon}{\Delta} \right)}^{\numOfIterations}} > 1 - \left(1 - \delta \right)^{\numOfMetaIterations}.
\]

Combining this fact with Lemma \ref{lem:lossOverfits}, and setting $\numOfMetaIterations = \left \lfloor \frac{1}{\delta} \right \rfloor $ we get on one hand,
\begin{align*}
\left| \expectation{\bar{\sampleSetRV} \sim \distDomainOfSets^{\numOfMetaIterations}, \left( \queryRV, \rangeRV, I \right) \sim \func{\text{Mon2}_{\bar{\mechanism}}}{\bar{\sampleSetRV}, \analyst}}{\func{\queryRV}{\distDomainOfSets} - \func{\queryRV}{\sampleSetRV_{I}}} \right| & \overset{\left( 1 \right)}{\ge} 2 \frac{\epsilon}{\Delta} \Delta \left(1 - \left(1 - \delta \right)^{\numOfMetaIterations} \right)
\\ & \overset{\left( 2 \right)}{>} 2 \epsilon \left(1 - e^{-\delta \left\lfloor \frac{1}{\delta} \right\rfloor} \right)
\\ & \overset{\left( 3 \right)}{>} 2 \epsilon \left(1 - e^{-1} \right),
\end{align*}
where (1) is a direct result of invoking Lemma \ref{lem:lossOverfits} with $1 - \left(1 - \delta \right)^{\numOfMetaIterations}$ for $\delta$, (2) uses the definition of $\numOfMetaIterations$ and (3) uses the inequality $1 - \delta \le e^{-\delta}$.

But on the other hand,
\begin{align*}
\left| \expectation{\bar{\sampleSetRV} \sim \distDomainOfSets^{\numOfMetaIterations}, \left( \queryRV, \rangeRV, I \right) \sim \func{\text{Mon2}_{\bar{\mechanism}}}{\bar{\sampleSetRV}, \analyst}}{\func{\queryRV}{\distDomainOfSets} - \func{\queryRV}{\sampleSetRV_{I}}} \right| & \overset{\left( 1 \right)}{\le} \left| \expectation{\bar{\sampleSetRV} \sim \distDomainOfSets^{\numOfMetaIterations}, \left( \queryRV, \rangeRV, I \right) \sim \func{\text{Mon2}_{\bar{\mechanism}}}{\bar{\sampleSetRV}, \analyst}}{\func{\queryRV}{\distDomainOfSets} - \rangeRV} \right|
\\ &~~~ + \left| \expectation{\bar{\sampleSetRV} \sim \distDomainOfSets^{\numOfMetaIterations}, \left( \queryRV, \rangeRV, I \right) \sim \func{\text{Mon2}_{\bar{\mechanism}}}{\bar{\sampleSetRV}, \analyst}}{\func{\queryRV}{\sampleSetRV_{I}} - \rangeRV} \right|
\\ & \overset{\left( 2 \right)}{<} \left(a \epsilon + 2 \numOfMetaIterations \cdot \frac{b \epsilon \delta}{\Delta} \Delta \right) + \left(c \epsilon + 2 \numOfMetaIterations \cdot \frac{d \epsilon \delta}{\Delta} \Delta \right)
\\ & \overset{\left( 3 \right)}{\le} \left(a + 2 b + c + 2 d \right) \epsilon
\\ & \overset{\left( 4 \right)}{<} 2 \epsilon \left(1 - e^{-1} \right),
\end{align*}
where (1) is the triangle inequality, (2) uses Claim \ref{clm:Monsampleerror} which was mentioned with relation to the original monitor mechanism (this time for the distribution error as well), (3) uses the definition of $\numOfMetaIterations$, and (4) the definition of $a, b, c, d$.

Since combining all of the above leads to a contradiction, we know that $\bar{\mechanism}$ cannot be $\tuple{\frac{\epsilon}{\Delta}}{\delta}$-$\numOfIterations$-LSS, which concludes the proof. The theorem was stated choosing $a = b = c = d = \frac{1}{5}$.
\end{proof}

\section{Relationship to other notions of stability} \label{sec:relToNotions}

In this section, we discuss the relationship between LSS and  a few common notions of stability.

\subsection{Definitions}

In the following definitions, $\domain, \distDomain, \queriesFamily, \range, \mechanism, \epsilon, \delta$ and $\sampleSize$ are used in a similar manner as for the definitions leading to LSS.
\begin{definition} [Differential Privacy \cite{DMNS06}] \label{def:diffPrivacy} 
Given $0 \le \epsilon$, $0 \le \delta \le 1$, and a query $\query$, a mechanism $\mechanism : \domainOfSets \times \queriesFamily \rightarrow \range$ will be called \emph{$\tuple{\epsilon}{\delta}$-differentially-private with respect to $\query$} (or \emph{DP}, for short) if for any $\sampleSet_1, \sampleSet_1 \in \domainOfSets$ that differ only in one element, the two distributions defined over $\range$ by $\func{\mechanism}{\sampleSet_1, \query}$ and $\func{\mechanism}{\sampleSet_2, \query}$ are $\tuple{\epsilon}{\delta}$-indistinguishable (Definition \ref{def:indist}). In other words, for any $\rangeSet \subseteq \range$,
\[
\distFuncDep{\rangeSet}{\sampleSet_1} \le e^\epsilon \cdot \distFuncDep{\rangeSet}{\sampleSet_2} + \delta,
\]
where the probability is taken over the internal randomness of $\mechanism$.
Notice that in this definition, there is no probabilistic aspect in the choice of $\sampleSet$, and the bound is defined on the worst case.
\end{definition}

\begin{definition} [Max Information \cite{DFHPRR15a}] \label{def:maxInfo} 
Given $0 \le \epsilon$, $0 \le \delta \le 1$, a distribution $\distDomainOfSets$, and a query $\query$, we say a mechanism $\mechanism : \domainOfSets \times \queriesFamily \rightarrow \range$ \emph{has $\delta$-approximate max-information of $\epsilon$ with respect to $\distDomainOfSets, \query$} (or \emph{MI}, for short) if the two distributions $\distJoint{\domainOfSets}{\range}{}$ and $\distProd{\domainOfSets}{\range}{}$ over $\domainOfSets \times \range$ are $\tuple{\epsilon}{\delta}$-indistinguishable. In other words, for any $\mathbf{b} \subseteq \domainOfSets \times \range$,
\[
\func{\distJoint{\domainOfSets}{\range}{}}{\mathbf{b}} \le e^{\epsilon} \cdot \func{\distProd{\domainOfSets}{\range}{}}{\mathbf{b}} + \delta \,\,\,\, \text{and} \,\,\,\,
\func{\distProd{\domainOfSets}{\range}{}}{\mathbf{b}} \le e^{\epsilon} \cdot \func{\distJoint{\domainOfSets}{\range}{}}{\mathbf{b}} + \delta.
\]
Some definitions replace $e$ with $2$ as the base of $\epsilon$.
\end{definition}

\begin{definition}[Typical Stability, based on Definition 2.3. of \cite{BF16}] \label{def:typStab} 
Given $0 \le \epsilon$, $0 \le \delta, \eta \le 1$, a distribution $\distDomainOfSets$, and a query $\query$, a mechanism $\mechanism : \domainOfSets \times \queriesFamily \rightarrow \range$ will be called \emph{$\left(\epsilon, \delta, \eta \right)$-Typically-Stable with respect to $\distDomainOfSets, \query$} (or \emph{TS}, for short) if with probability at least $1 - \eta$ over the sampling of $\sampleSet_{1}, \sampleSet_{2} \in \domainOfSets$, the conditional distributions induced by the mechanism given the two sets is $\tuple{\epsilon}{\delta}$-indistinguishable. Formally,
\[
\prob{\sampleSetRV_{1}, \sampleSetRV_{2} \sim \distDomainOfSets}{\exists \rangeSet \subseteq \range \,|\, \distFuncDep{\rangeSet}{\sampleSetRV_{1}} > e^{\epsilon} \distFuncDep{\rangeSet}{\sampleSetRV_{2}} + \delta } < \eta
\]
An equivalent definition requires the existence of a subset $\sampleSetSet \in \domainOfSets$, such that (1) $\distFunc{\sampleSetSet} \ge 1 - \eta$, and (2) for any $\sampleSet_{1}, \sampleSet_{2} \in \sampleSetSet$ 
\[
\distFuncDep{\rangeSet}{\sampleSet_1} \le e^\epsilon \cdot \distFuncDep{\rangeSet}{\sampleSet_2} + \delta
\]
\end{definition}

Notice that in a way, MI and TS are a natural relaxation of DP, where instead of considering only the probability which is induced by the mechanism, we also consider the underlying distribution.

\begin{definition} [Local Max Information] \label{def:lclMaxInf}
Given $0 \le \epsilon$, $0 \le \delta \le 1$, a distribution $\distDomainOfSets$ and a query $\query$, a mechanism $\mechanism$ will be said to satisfy \emph{$\tuple{\epsilon}{\delta}$-Local-Max-Information with respect to $\distDomainOfSets$ and $\query$} (or \emph{LMI}, for short), if the joint distributions $\distJoint{\domain}{\range}{}$ and the product distribution $\distProd{\domain}{\range}{}$ over $\domain \times \range$ are $\tuple{\epsilon}{\delta}$-indistinguishable. In other words, for any $\mathbf{b} \subseteq \domain \times \range$,
\[
\func{\distJoint{\domain}{\range}{}}{\mathbf{b}} \le e^{\epsilon} \cdot \func{\distProd{\domain}{\range}{}}{\mathbf{b}} + \delta \,\,\,\, \text{and} \,\,\,\,
\func{\distProd{\domain}{\range}{}}{\mathbf{b}} \le e^{\epsilon} \cdot \func{\distJoint{\domain}{\range}{}}{\mathbf{b}} + \delta.
\]
\end{definition}

In a way, this definition takes the relaxation one step further, by moving from the distribution over the sample sets to distribution over the sample elements. Unlike the previous definitions, this one was not presented yet as far as we know. This holds for definition \ref{def:bndLclMaxLkg} as well.

These definition can be extended to apply to a family of queries and/or a family of possible distributions, just like the LSS definition.

\begin{definition}[Bounded Maximal Leakage \cite{EGI19}] \label{def:bndMaxLkg} 
Given $0 \le \epsilon$, a distribution $\distDomainOfSets$, and a query $\query$, a mechanism $\mechanism : \domainOfSets \times \queriesFamily \rightarrow \range$ will be called \emph{$\epsilon$-Bounded-Maximal-Leaking with respect to $\distDomainOfSets, \query$} (or \emph{ML}, for short) if $\func{\mathcal{L}}{\distDomainOfSets \rightarrow \distRange{}} \le \epsilon$, where $\mathcal{L}$ is the Maximal Leakage (Definition \ref{dfn:maxLkg}).
\end{definition}

Similarly to MI, this definition can also be relaxed to the local version.

\begin{definition}[Bounded Local Maximal Leakage] \label{def:bndLclMaxLkg} 
Given $0 \le \epsilon$, a distribution $\distDomainOfSets$, and a query $\query$, a mechanism $\mechanism : \domainOfSets \times \queriesFamily \rightarrow \range$ will be called \emph{$\epsilon$-Bounded-Local-Maximal-Leaking with respect to $\distDomainOfSets, \query$} (or \emph{ML}, for short) if $\func{\mathcal{L}}{\distDomain \rightarrow \distRange{}} \le \epsilon$, where $\mathcal{L}$ is the Maximal Leakage (Definition \ref{dfn:maxLkg}).
\end{definition}

\begin{definition}[Compression Scheme \cite{LW86}] \label{def:ComprSchm} 
Given an integer $m < \frac{\sampleSize}{2}$ and a query $\query$, a mechanism $\mechanism$ will be said to have a \emph{compression scheme} of size $m$ with respect to $\query$ (or \emph{CS}, for short), if $\mechanism$ can be described as the composition $f_{\query} \circ g_{\query}$ where the \emph{compression function} $g_{\query} : \domainOfSets \rightarrow \domain^{m}$ has the property that $\func{g_{\query}}{\sampleSet} \subset \sampleSet$ and $f_{\query} : \domain^{m} \rightarrow \range$ is some arbitrary function which will be called the \emph{encoding function}. Both functions might be non deterministic. We will denote $w \coloneqq \func{g}{\sampleSet}$ and $r_{w} \coloneqq \func{f}{w}$.\footnote{some versions include the option of receiving some side information, i.e. the coin tosses of $g$.}

One simple case is when $f$ is the identity function, and the mechanism releases $m$ sample elements.
\end{definition}

\subsection{Implications}
Prior work (\cite{DFHPRR15a} and \cite{RRST16}) showed that bounded DP implies bounded MI. In the case of $\delta > 0$, this holds only if the underlying distribution is a product distribution~\cite{De12}). Bounded MI is also implied by TS \cite{BF16}, and ML \cite{EGI19}. We  prove that DP, MI and TS imply LMI (in the case of DP, only for product distributions). All proofs for this subsection can be found in Appendix \ref{sec:implicationsProofs}.

\begin{theorem}[Differential Privacy implies Local Max Information]\label{thm:DPimpLMI}
Given $0 \le \epsilon$, $0 \le \delta \le 1$, a distribution $\distDomain$, and a query $\query$, if a mechanism $\mechanism$ is $\tuple{\epsilon}{\delta}$-DP with respect to $\query$ then it is $\tuple{\epsilon}{\delta}$-LMI with respect to the same $\query$ and the product distribution over $\domainOfSets$ induced by $\distDomain$.
\end{theorem}

\begin{theorem}[Max Information implies Local Max Information]\label{thm:MIimpLMI}
Given $0 \le \epsilon$, $0 \le \delta \le 1$, a distribution $\distDomainOfSets$ and a query $\query$, if a mechanism $\mechanism$ has $\delta$-approximate max-information of $\epsilon$ with respect to $\distDomainOfSets$ and $\query$ then it is $\tuple{\epsilon}{\delta}$-LMI with respect to the same $\distDomainOfSets$ and $\query$.
\end{theorem}

\begin{theorem}[Typical Stability implies Local Max Information]\label{thm:TSimpLMI}
Given $0 \le \epsilon$, $0 \le \delta, \eta \le 1$, a distribution $\distDomainOfSets$ and a query $\query$, if a mechanism $\mechanism$ is $\left( \epsilon, \delta, \eta \right)$-Typically Stable with respect to $\distDomainOfSets$ and $\query$ then it is $\tuple{\epsilon}{\delta + 2 \eta}$-LMI with respect to the same $\distDomainOfSets$ and $\query$.
\end{theorem}

These three theorems follow naturally from the fact that LMI is a fairly direct relaxation of DP, MI and TS. We also show that LML implies LMI.

\begin{theorem} [Local Bounded Maximal Leakage implies Local Max Information] \label{thm:LMLimpLMI}
Given $0 \le \epsilon$, $0 < \delta \le 1$ a distribution $\distDomainOfSets$ and a query $\query$, if a mechanism $\mechanism$ is $\epsilon$-LML with respect to $\distDomainOfSets$ and $\query$, then it is $\tuple{\epsilon + \func{\ln}{\frac{1}{\delta}}}{\delta}$-LMI with respect to the same $\distDomainOfSets$ and $\query$.
\end{theorem}

We next show that LMI implies LSS.

\begin{theorem} [Local Max Information implies Local Statistical Stability] \label{thm:LMIimpLSS}
Given $0 \le \delta \le \epsilon \le \frac{1}{3}$, a distribution $\distDomainOfSets$ and a query $\query$, if a mechanism $\mechanism$ is $\tuple{\epsilon}{\delta}$-LMI with respect to $\distDomainOfSets$ and $\query$, then it is $\tuple{\epsilon'}{\frac{\delta}{\epsilon}}$-LSS with respect to the same $\distDomainOfSets$ and $\query$, where $\epsilon' = e^{\epsilon} - 1 + \epsilon$.
\end{theorem}

We also prove that Compression Schemes (Definition \ref{def:ComprSchm}) imply LSS. This results from the fact that releasing information based on a restricted number of sample elements has a limited effect on the posterior distribution on one element of the sample set.

\begin{theorem}[Compressibility implies Local Statistical Stability] \label{thm:CSimpLSS}
Given $0 \le \delta \le 1$, an integer $m \le \frac{\sampleSize}{9 \func{\ln}{2 \sampleSize / \delta}}$, a distribution $\distDomain$, and a query $\query \in \queriesFamily$, if a mechanism $\mechanism$ has a compression scheme of size $m$ then it is $\tuple{\epsilon}{\delta}$-LSS with respect to the same $\query$ and the product distribution over $\domainOfSets$ induced by $\distDomain$, for any $\epsilon > 11 \sqrt{\frac{m \func{\ln}{2 \sampleSize / \delta}}{\sampleSize}}$.\footnote{In case $g$ releases some side information, the number of bits required to describe this information is added to the $m$ factor in the bound on $\epsilon$.}
\end{theorem}

\subsection{Separations}
Finally, we show that MI is a strictly stronger requirement than LMI, and LMI is a strictly stronger requirement then LSS. Proofs of these theorems appear in Appendix~\ref{sec:separationsProofs}.

\begin{theorem} [Max Information is strictly stronger than Local Max Information] \label{thm:MIgrLMI}
For any $0 < \epsilon$, $\sampleSize \ge 3$, the mechanism which outputs the parity function of the sample set is $\tuple{\epsilon}{0}$-LMI but not $\tuple{1}{\frac{1}{5}}$-MI.
\end{theorem}

\begin{theorem} [Local Max Information is strictly stronger than Local Statistical Stability] \label{thm:LMIgrLSS}
For any $0 \le \delta \le 1$, $\sampleSize > \max \left\{ 2 \func{\ln}{\frac{2}{\delta}}, 6\right\}$, a mechanism which uniformly samples and outputs one sample element is $\tuple{11 \sqrt{\frac{\func{\ln}{2 \sampleSize / \delta}}{\sampleSize}}}{\delta}$-LSS but is not $\tuple{1}{\frac{1}{2 \sampleSize}}$-LMI.
\end{theorem}

\section{Applications and Discussion} \label{sec:appAndDisc}
In order to make the LSS notion useful, we must identify mechanisms which manages to remain stable while maintaining sample accuracy. Fortunately, many such mechanisms have been introduced in the context of Differential Privacy. Two of the most basic Differentially Private mechanisms are based on noise addition, of either a Laplace or a Gaussian random variable. Careful tailoring of their parameters allows ``masking'' the effect of changing one element, while maintaining a limited effect on the sample accuracy. By Theorems \ref{thm:DPimpLMI} and \ref{thm:LMIimpLSS}, these mechanisms are guaranteed to be LSS as well. The definitions and properties of these mechanisms can be found in Appendix \ref{apd:appAndDisc}.

In moving away from the study of worst-case data sets (as is common in previous stability notions) to averaging over sample sets and over data elements of those sets, we hope that the Local Statistical Stability notion will enable new progress in the study of generalization under adaptive data analysis. This averaging, potentially leveraging a sort of ``natural noise'' from the data sampling process, may enable the development of new algorithms to preserve generalization, and may also support tighter bounds on the implications of existing algorithms. One possible way this might be achieved is by limiting the family of distributions and queries, such that the empirical mean of the query lies within some confidence interval around population mean, which would allow scaling the noise to the interval rather than the full range (see, e.g. , \emph{Concentrated Queries}, as proposed by \cite{BF16}).

One might also hope that realistic adaptive learning settings are not adversarial, and might therefore enjoy even better generalization guarantees.
LSS may be a tool for understanding the generalization properties of algorithms of interest (as opposed to worst-case queries or analysts; see e.g. \cite{GK16}, \cite{ZH19}).

\paragraph{Acknowledgements}
This work was supported in part by Israel Science Foundation (ISF) grant 1044/16, the United States Air Force and DARPA under contract FA8750-16-C-0022, and the Federmann Cyber Security Center in conjunction with the Israel national cyber directorate. Part of this work was done while the authors were visiting the Simons Institute for the Theory of Computing.
Any opinions,
findings and conclusions or recommendations expressed in this material
are those of the authors and do not necessarily reflect the views of
the United States Air Force and DARPA.

\bibliographystyle{alpha}
\bibliography{bibliography}

\newpage
\appendix

\section{Distributions: Formal Definitions} \label{apd:distDef}

\begin{definition} [Distributions over $\domainOfSets$ and $\range$] \label{def:distsOfSets}
A distribution $\distDomainOfSets$, a query $\query$, and a mechanism $\mechanism : \domainOfSets \times \queriesFamily \rightarrow \range$, together induce a set of distributions over $\domainOfSets$, $\range$, and $\domainOfSets \times \range$.

The \emph{conditional distribution} $\distDep{\range}{\domainOfSets}{\query}$ over $\range$ represents the probability to get $\rangeVal$ as the output of $\mechanismFunc{\sampleSet, \query}$. That is, $\forall \sampleSet \in \domainOfSets, \rangeVal\in\range,$
\[
\func{\distDep{\range}{\domainOfSets}{\query}}{\rangeVal\,|\,\sampleSet} \coloneqq \prob{\rangeRV \sim \mechanismFunc{\sampleSet, \query}}{\rangeRV = \rangeVal \,|\, \sampleSet},
\]
where the probability is taken over the internal randomness of $\mechanism$.

The \emph{joint distribution} $\distJoint{\domainOfSets}{\range}{\query}$ over $\domainOfSets \times \range$ represents the probability to sample a particular $\sampleSet$ and get $\rangeVal$ as the output of $\mechanismFunc{\sampleSet, \query}$. That is, $\forall \sampleSet \in \domainOfSets, \rangeVal \in \range$,
\[
\func{\distJoint{\domainOfSets}{\range}{\query}}{\sampleSet, \rangeVal} \coloneqq \func{\distDomainOfSets}{\sampleSet} \cdot \func{\distDep{\range}{\domainOfSets}{\query}}{\rangeVal \,|\, \sampleSet}.
\]

The \emph{marginal distribution} $\distRange{\query}$ over $\range$ represents the prior probability to get output $\rangeVal$ without any knowledge of $\sampleSet$. That is, $\forall \rangeVal \in \range,$
\[
\func{\distRange{\query}}{\rangeVal} \coloneqq
\sum_{\sampleSet \in \domainOfSets} \func{\distJoint{\domainOfSets}{\range}{\query}}{\sampleSet, \rangeVal}.
\]

The \emph{product distribution} $\distProd{\domainOfSets}{\range}{\query}$ over $\domainOfSets \times \range$ represents the probability to sample $\sampleSet$ and get $\rangeVal$ as the output of $\mechanismFunc{\cdot, \query}$ independently. That is, $\forall \sampleSet \in \domain, \rangeVal \in \range,$
\[
\func{\distProd{\domainOfSets}{\range}{\query}}{\sampleSet, \rangeVal} \coloneqq \func{\distDomainOfSets}{\sampleSet} \cdot \func{\distRange{\query}}{\rangeVal}.
\]

The \emph{conditional distribution} $\distDep{\domainOfSets}{\range}{\query}$ over $\domainOfSets$ represents the posterior probability that the sample set was $\sampleSet$ given that $\mechanismFunc{\cdot, \query}$ returns $\rangeVal$. That is, $\forall \sampleSet \in \domainOfSets, \rangeVal\in\range,$
\[
\func{\distDep{\domainOfSets}{\range}{\query}}{\sampleSet \,|\, \rangeVal} \coloneqq \frac{\func{\distJoint{\domainOfSets}{\range}{\query}}{\sampleSet, \rangeVal}}{\func{\distRange{q}}{\rangeVal}}.
\]
\end{definition}

\begin{definition} [Distributions over $\domain$ and $\range$] \label{def:distsOfDomain}
The \emph{marginal distribution} $\distDomain$ over $\domain$ represents the probability to get $\domainVal$ by sampling a sample set uniformly at random without any knowledge of $\sampleSet$. That is, $\forall \domainVal \in \domain,$
\[
\func{\distDomain}{\domainVal} \coloneqq \sum_{\sampleSet \in \domainOfSets} \func{\distDomainOfSets}{\sampleSet} \cdot \func{\distDep{\domain}{\domainOfSets}{}}{\domainVal \,|\, \sampleSet},
\]
where $\func{\distDep{\domain}{\domainOfSets}{}}{\domainVal \,|\, \sampleSet}$ denotes the probability to get $\domainVal$ by sampling $\sampleSet$ uniformly at random.

The \emph{joint distribution} $\distJoint{\domain}{\range}{\query}$ over $\domain\times\range$ represents the probability to get $\domainVal$ by sampling a sample set uniformly at random and also get $\rangeVal$ as the output of $\mechanismFunc{\cdot, \query}$ from the same sample set. That is, $\forall \domainVal\in\domain,\rangeVal\in\range,$
\[
\func{\distJoint{\domain}{\range}{\query}}{\domainVal,\rangeVal} \coloneqq \sum_{\sampleSet \in \domainOfSets} \func{\distDomainOfSets}{\sampleSet} \cdot \func{\distDep{\domain}{\domainOfSets}{}}{\domainVal \,|\, \sampleSet} \cdot \func{\distDep{\range}{\domainOfSets}{\query}}{\rangeVal \,|\, \sampleSet}.
\]
where $\func{\distDep{\domain}{\domainOfSets}{}}{\domainVal \,|\, \sampleSet}$ denotes the probability to get $\domainVal$ by sampling $\sampleSet$ uniformly at random.

The \emph{product distribution} $\distProd{\domain}{\range}{\query}$ over $\domain\times\range$ represents the probability to get $\domainVal$ by sampling a sample set uniformly at random and get $\rangeVal$ as the output of $\mechanismFunc{\cdot, \query}$ independently. That is, $\forall \domainVal\in\domain,\rangeVal\in\range,$
\[
\func{\distProd{\domain}{\range}{\query}}{\domainVal,\rangeVal} \coloneqq \func{\distUpInd{\domain}{}}{\domainVal} \cdot \func{\distRange{\query}}{\rangeVal}.
\]

The \emph{conditional distribution} $\distDep{\range}{\domain}{\query}$ over $\range$ represents the probability to get $\rangeVal$ as the output of $\mechanismFunc{\cdot, \query}$ from a sample set, given the fact that we got $\domainVal$ by sampling the same sample set uniformly at random. That is, $\forall \domainVal \in \domain, \rangeVal\in\range,$
\[
\func{\distDep{\range}{\domain}{\query}}{\rangeVal\,|\,\domainVal} \coloneqq \sum_{\sampleSet \in \domainOfSets} \func{\distDep{\domainOfSets}{\domain}{}}{\sampleSet \,|\, \domainVal} \cdot \func{\distDep{\range}{\domainOfSets}{\query}}{\rangeVal\,|\,\sampleSet}.
\]

The \emph{conditional distribution} $\distDep{\domain}{\range}{\query}$ over $\domain$ represents the probability to get $\domainVal$ by sampling a sample set uniformly at random, given the fact that we got $\rangeVal$ as the output of $\mechanismFunc{\cdot, \query}$ from that sample set. That is, $\forall \domainVal \in \domain, \rangeVal\in\range,$
\[
\func{\distDep{\domain}{\range}{\query}}{\domainVal \,|\, \rangeVal} \coloneqq \sum_{\sampleSet \in \domainOfSets} \func{\distDep{\domainOfSets}{\range}{\query}}{\sampleSet \,|\, \rangeVal} \cdot \func{\distDep{\domain}{\domainOfSets}{}}{\domainVal\,|\,\sampleSet}.
\]
\end{definition}
Although all of these definitions depend on $\distDomainOfSets$ and $\mechanism$, we typically omit these from the notation for simplicity, and usually omit the superscripts and subscripts entirely. We include them only when necessary for clarity.
We also use $\distInd{}$ to denote the probability of a set: for $\rangeSet \subseteq \range$, we define $\func{\distRange{\query}}{\rangeSet} \coloneqq \underset{\rangeVal \in \rangeSet}{\sum} \func{\distRange{\query}}{\rangeVal}$.

Though the conditional distributions $\distDep{\range}{\domain}{\query}$ and $\distDep{\domain}{\range}{\query}$ were not defined as the ratio between the joint and marginal distribution, the analogue of Bayes' rule still holds for these distributions.
\begin{proposition} \label{lem:depDist} Given any distribution $\distDomainOfSets$, mechanism $\mechanism : \domainOfSets \times \queriesFamily \rightarrow \range$, and query $\query$,
\[
\func{\distJoint{\domain}{\range}{}}{\domainVal,\rangeVal} = \distFunc{\domainVal} \cdot \distFuncDep{\rangeVal}{\domainVal} = \distFunc{\rangeVal} \cdot \distFuncDep{\domainVal}{\rangeVal}.
\]
\end{proposition}

\begin{proof}
We observe
\begin{align*}
\func{\distJoint{\domain}{\range}{}}{\domainVal,\rangeVal} & = \sum_{\sampleSet \in \domainOfSets} \distFunc{\sampleSet} \cdot \distFuncDep{\domainVal}{\sampleSet} \cdot \distFuncDep{\rangeVal}{\sampleSet}
\\ & = \sum_{\sampleSet \in \domainOfSets} \func{\distJoint{\domainOfSets}{\domain}{}}{\sampleSet, \domainVal} \cdot \distFuncDep{\rangeVal}{\sampleSet}
\\ & = \distFunc{\domainVal} \cdot \sum_{\sampleSet \in \domainOfSets} \frac{\func{\distJoint{\domainOfSets}{\domain}{}}{\sampleSet, \domainVal}}{\distFunc{\domainVal}} \cdot \distFuncDep{\rangeVal}{\sampleSet}
\\ & = \distFunc{\domainVal} \cdot \sum_{\sampleSet \in \domainOfSets} \distFuncDep{\sampleSet}{\domainVal} \cdot \distFuncDep{\rangeVal}{\sampleSet}
\\ & = \distFunc{\domainVal} \cdot \distFuncDep{\rangeVal}{\domainVal}.
\end{align*}

Similarly,
\begin{align*}
\func{\distJoint{\domain}{\range}{}}{\domainVal,\rangeVal} & = \sum_{\sampleSet \in \domainOfSets} \distFunc{\sampleSet} \cdot \distFuncDep{\rangeVal}{\sampleSet} \cdot \distFuncDep{\domainVal}{\sampleSet}
\\ & = \sum_{\sampleSet \in \domainOfSets} \func{\distJoint{\domainOfSets}{\range}{}}{\sampleSet, \rangeVal} \cdot \distFuncDep{\domainVal}{\sampleSet}
\\ & = \distFunc{\rangeVal} \cdot \sum_{\sampleSet \in \domainOfSets} \frac{\func{\distJoint{\domainOfSets}{\range}{}}{\sampleSet, \rangeVal}}{\distFunc{\rangeVal}} \cdot \distFuncDep{\domainVal}{\sampleSet}
\\ & = \distFunc{\rangeVal} \cdot \sum_{\sampleSet \in \domainOfSets} \distFuncDep{\sampleSet}{\rangeVal} \cdot \distFuncDep{\domainVal}{\sampleSet}
\\ & = \distFunc{\rangeVal} \cdot \distFuncDep{\domainVal}{\rangeVal}.
\end{align*}
\end{proof}
\section{Missing Details from Section~\ref{sec:defns}}

\subsection{Post-Processing} \label{sec:postProcProof}
\begin{proof}[Missing parts from the proof of Theorem~\ref{thm:postprocessing}]
\begin{align*}
\sum_{\mappedRangeVal \in \mappedRangeSet_{\epsilon}} \distFunc{\mappedRangeVal} & = \sum_{\mappedRangeVal \in \mappedRangeSet_{\epsilon}} \sum_{\rangeVal \in \range} \distFunc{\rangeVal} \cdot \distFuncDep{\mappedRangeVal}{\rangeVal}
\\ & = \sum_{\rangeVal \in \range} \overbrace{\sum_{\mappedRangeVal \in \mappedRangeSet_{\epsilon}} \distFuncDep{\mappedRangeVal}{\rangeVal}}^{= \func{\weight_{\mappedRange}^{\epsilon}}{\rangeVal}} \cdot \distFunc{\rangeVal} 
\\ & = \sum_{\rangeVal \in \range} \func{\weight_{\mappedRange}^{\epsilon}}{\rangeVal} \cdot \distFunc{\rangeVal},
\end{align*}

\begin{align*}
\sum_{\mappedRangeVal \in \mappedRangeSet_{\epsilon}} \distFunc{\mappedRangeVal} \cdot \func{\ell}{\mappedRangeVal} & = \sum_{\mappedRangeVal \in \mappedRangeSet_{\epsilon}} \distFunc{\mappedRangeVal} \sum_{\domainVal \in \func{\domainSet_{+}}{\mappedRangeVal}} \left( \distFuncDep{\domainVal}{\mappedRangeVal} - \distFunc{\domainVal} \right)
\\ & = \sum_{\mappedRangeVal \in \mappedRangeSet_{\epsilon}} \sum_{\domainVal \in \func{\domainSet_{+}}{\mappedRangeVal}} \distFunc{\domainVal} \left( \distFuncDep{\mappedRangeVal}{\domainVal} - \distFunc{\mappedRangeVal} \right)
\\ & = \sum_{\mappedRangeVal \in \mappedRangeSet_{\epsilon}} \sum_{\rangeVal \in \range} \sum_{\domainVal \in \func{\domainSet_{+}}{\mappedRangeVal}} \distFunc{\domainVal} \left( \distFuncDep{\rangeVal}{\domainVal} - \distFunc{\rangeVal} \right) \distFuncDep{\mappedRangeVal}{\rangeVal}
\\ & \overset{\left( 1 \right)}{\le} \sum_{\mappedRangeVal \in \mappedRangeSet_{\epsilon}} \sum_{\rangeVal \in \range} \sum_{\domainVal \in \func{\domainSet_{+}}{\rangeVal}} \distFunc{\domainVal} \left( \distFuncDep{\rangeVal}{\domainVal} - \distFunc{\rangeVal} \right) \distFuncDep{\mappedRangeVal}{\rangeVal}
\\ & = \sum_{\rangeVal \in \range} \overbrace{\sum_{\mappedRangeVal \in \mappedRangeSet_{\epsilon}} \distFuncDep{\mappedRangeVal}{\rangeVal}}^{= \func{\weight_{\mappedRange}^{\epsilon}}{\rangeVal}} \cdot \distFunc{\rangeVal} \sum_{\domainVal \in \func{\domainSet_{+}}{\rangeVal}} \left( \distFuncDep{\domainVal}{\rangeVal} - \distFunc{\domainVal} \right)
\\ & = \sum_{\rangeVal \in \rangeSet} \func{\weight_{\mappedRange}^{\epsilon}}{\rangeVal} \cdot \distFunc{\rangeVal} \cdot \func{\ell}{\rangeVal},
\end{align*}
where (1) results from the definition of $\func{\domainSet_{+}}{\rangeVal}$.
\end{proof}

\subsection{Adaptivity and View-Induced Posterior Distributions} \label{sec:adapCompProof}

\begin{proof} [Proof of Lemma \ref{lem:liniarityOfLoss}]
We begin by proving a set of relations between the prior distributions over $\view_{\numOfIterations + 1}$ and the posterior distributions induced by the view $\viewVal_{\numOfIterations }$.
\begin{align*}
\func{\distJoint{\domainOfSets}{\view_{\numOfIterations + 1}}{}}{\sampleSet, \viewVal_{\numOfIterations + 1}} & = \distFunc{\sampleSet} \cdot \distFuncDep{\viewVal_{\numOfIterations + 1}}{\sampleSet}
\\ & \overset{\left(1\right)}{=} \distFunc{\sampleSet} \cdot \distFuncDep{\viewVal_{\numOfIterations}}{\sampleSet} \cdot \func{\distUpInd{}{\query_{\numOfIterations + 1}}}{\rangeVal_{\numOfIterations + 1} \,|\, \sampleSet, \viewVal_{\numOfIterations}}
\\ & = \distFunc{\viewVal_{\numOfIterations}} \cdot \distFuncDep{\sampleSet}{\viewVal_{\numOfIterations}} \cdot \func{\distUpInd{}{\query_{\numOfIterations + 1}}}{\rangeVal_{\numOfIterations + 1} \,|\, \sampleSet, \viewVal_{\numOfIterations}}
\\ & = \distFunc{\viewVal_{\numOfIterations}} \cdot \func{\distPUpInd{}{\viewVal_{\numOfIterations}}}{\sampleSet} \cdot \distPFuncDep{\viewVal_{\numOfIterations}}{\rangeVal_{\numOfIterations + 1}}{\sampleSet}
\\ & = \distFunc{\viewVal_{\numOfIterations}} \cdot \func{\distPJoint{\domainOfSets}{\range}{\viewVal_{\numOfIterations}}}{\sampleSet, \rangeVal_{\numOfIterations + 1}},
\end{align*}
where (1) is a result of the fact that $\query_{\numOfIterations + 1}$ is a deterministic function of $\viewVal_{\numOfIterations}$. As mentioned in Definition \ref{def:posteriorDist}, the distribution of $\rangeVal_{\numOfIterations + 1}$ might depend on $\viewVal_{\numOfIterations}$ in the case of a stateful mechanism, but it is all encapsulated in the definition of $\distP{}$.

Using this identity and the definition of $\distPDomainOfSets{\viewVal_{\numOfIterations}}$ we get that,
\[
\distFunc{\viewVal_{\numOfIterations + 1}} 
= \sum_{\sampleSet \in \domainOfSets} \func{\distJoint{\domainOfSets}{\view_{\numOfIterations + 1}}{}}{\sampleSet, \viewVal_{\numOfIterations + 1}}
= \sum_{\sampleSet \in \domainOfSets} \distFunc{\viewVal_{\numOfIterations}} \cdot \func{\distPJoint{\domainOfSets}{\range}{\viewVal_{\numOfIterations}}}{\sampleSet, \rangeVal_{\numOfIterations + 1}}
= \distFunc{\viewVal_{\numOfIterations}} \cdot \distPFunc{\viewVal_{\numOfIterations}}{\rangeVal_{\numOfIterations + 1}}.
\]
\[
\distFuncDep{\domainVal}{\viewVal_{\numOfIterations}} 
= \sum_{\sampleSet \in \domainOfSets} \distFuncDep{\sampleSet}{\viewVal_{\numOfIterations}} \cdot \distFuncDep{\domainVal}{\sampleSet}
= \sum_{\sampleSet \in \domainOfSets} \distPFunc{\viewVal_{\numOfIterations}}{\sampleSet} \cdot \distFuncDep{\domainVal}{\sampleSet}
= \distPFunc{\viewVal_{\numOfIterations}}{\domainVal}.
\]
\[
\distFuncDep{\domainVal}{\viewVal_{\numOfIterations + 1}}
= \sum_{\sampleSet \in \domainOfSets} \distFuncDep{\sampleSet}{\viewVal_{\numOfIterations + 1}} \cdot \distFuncDep{\domainVal}{\sampleSet}
= \sum_{\sampleSet \in \domainOfSets} \distPFuncDep{\viewVal_{\numOfIterations}}{\sampleSet}{\rangeVal_{\numOfIterations + 1}} \cdot \distFuncDep{\domainVal}{\sampleSet}
= \distPFuncDep{\viewVal_{\numOfIterations}}{\domainVal}{\rangeVal_{\numOfIterations + 1}}.
\]

where we keep using the fact that $\distFuncDep{\domainVal}{\sampleSet}$ does not depend on the underlying distribution $\distDomainOfSets$ at all. Using these identities we can analyze the stability loss, and we would do so by invoking an equivalent definition of the statistical distance (see Appendix \ref{apd:DistMeasures}),

\begin{align*}
\func{\ell_{\distDomainOfSets}^{\analyst}}{\viewVal_{\numOfIterations + 1}} & =  \frac{1}{2} \sum_{\domainVal \in \domain} \left| \distFuncDep{\domainVal}{\viewVal_{\numOfIterations + 1}} - \func{\distDomain}{\domainVal} \right|
\\ & \le^{\left(1\right)}  \frac{1}{2} \sum_{\domainVal \in \domain} \left| \distFuncDep{\domainVal}{\viewVal_{\numOfIterations}} - \func{\distDomain}{\domainVal} \right| + \frac{1}{2} \sum_{\domainVal \in \domain} \left| \distFuncDep{\domainVal}{\viewVal_{\numOfIterations + 1}} - \distFuncDep{\domainVal}{\viewVal_{\numOfIterations}} \right|
\\ & =  \func{\ell_{\distDomainOfSets}^{\analyst}}{\viewVal_{\numOfIterations}} + \frac{1}{2} \sum_{\domainVal \in \domain} \left| \func{\distPDep{\domain}{\range}{\viewVal_{\numOfIterations}, \query_{\numOfIterations + 1}}}{\domainVal \,|\, \rangeVal_{\numOfIterations + 1}} - \func{\distPDomain{\viewVal_{\numOfIterations}}}{\domainVal} \right|
\\ & =  \func{\ell_{\distDomainOfSets}^{\analyst}}{\viewVal_{\numOfIterations}} + \func{\ell_{\distPDomainOfSets{\viewVal_{\numOfIterations}}}^{\query_{\numOfIterations + 1}}}{\rangeVal_{\numOfIterations + 1}},
\end{align*}
where (1) is simply the triangle inequality.
\end{proof}

\begin{proof}[Missing parts from the proof of Theorem \ref{thm:adaptiveComposition}]
\begin{align*}
\distFunc{\viewSet_{\epsilon_{\left[ \numOfIterations + 1 \right]}}^{\numOfIterations + 1}} & \cdot \left( \func{\ell_{\distDomainOfSets}^{\analyst}}{\viewSet_{\epsilon_{\left[ \numOfIterations + 1 \right]}}^{\numOfIterations + 1}} - \epsilon_{\left[ \numOfIterations + 1 \right]} \right)
\\ & = \sum_{\viewVal_{\numOfIterations + 1} \in \viewSet_{\epsilon_{\left[ \numOfIterations + 1 \right]}}^{\numOfIterations + 1}} \distFunc{\viewVal_{\numOfIterations + 1}} \left( \func{\ell_{\distDomainOfSets}^{\analyst}}{\viewVal_{\numOfIterations + 1}} - \epsilon_{\left[ \numOfIterations + 1 \right]} \right)
\\ & \overset{\left(1\right)}{\le}  \sum_{\tuple{\viewVal_{\numOfIterations}}{\rangeVal_{\numOfIterations + 1}} \in \viewSet_{\epsilon_{\left[ \numOfIterations + 1 \right]}}^{\numOfIterations + 1}} \distFunc{\viewVal_{\numOfIterations}} \cdot \distPFunc{\viewVal_{\numOfIterations}}{\rangeVal_{\numOfIterations + 1}} \left( \func{\ell_{\distDomainOfSets}^{\analyst}}{\viewVal_{\numOfIterations}} + \func{\ell_{\distPDomainOfSets{\viewVal_{\numOfIterations}}}^{\query_{\numOfIterations + 1}}}{\rangeVal_{\numOfIterations + 1}} - \epsilon_{\left[ \numOfIterations + 1 \right]} \right)
\end{align*}
where (1) is a direct result of Lemma \ref{lem:liniarityOfLoss}. Analyzing the two parts separately we get

\begin{align*}
\sum_{\viewVal_{\numOfIterations} \in \view_{\numOfIterations}} \sum_{\rangeVal_{\numOfIterations + 1} \in \func{\rangeSet}{\viewVal_{\numOfIterations}}} \distFunc{\viewVal_{\numOfIterations}} \cdot \distPFunc{\viewVal_{\numOfIterations}}{\rangeVal_{\numOfIterations + 1}} \left( \func{\ell_{\distDomainOfSets}^{\analyst}}{\viewVal_{\numOfIterations}} - \epsilon_{\left[ \numOfIterations \right]} \right) & \overset{\left( 1 \right)}{\le} \sum_{\viewVal_{\numOfIterations} \in \viewSet_{\epsilon_{\left[ \numOfIterations \right]}}^{\numOfIterations}} \distFunc{\viewVal_{\numOfIterations}} \left( \func{\ell_{\distDomainOfSets}^{\analyst}}{\viewVal_{\numOfIterations}} - \epsilon_{\left[ \numOfIterations \right]} \right)
\\ & = \distFunc{\viewSet_{\epsilon_{\left[ \numOfIterations \right]}}^{\numOfIterations}} \left( \func{\ell_{\distDomainOfSets}^{\analyst}}{\viewSet_{\epsilon_{\left[ \numOfIterations \right]}}^{\numOfIterations}} - \epsilon_{\left[ \numOfIterations \right]} \right)
\\ & \overset{\left( 2 \right)}{\le} \sum_{i \in \left[ \numOfIterations \right]} \delta_{i}
\end{align*}
and similarly,

\begin{align*}
\sum_{\rangeVal_{\numOfIterations + 1} \in \range} \sum_{\viewVal_{\numOfIterations} \in \func{\viewSet_{\numOfIterations}}{\rangeVal_{\numOfIterations + 1}}} \distFunc{\viewVal_{\numOfIterations}} \cdot \distPFunc{\viewVal_{\numOfIterations}}{\rangeVal_{\numOfIterations + 1}} & \left( \func{\ell_{\distPDomainOfSets{\viewVal_{\numOfIterations}}}^{\query_{\numOfIterations + 1}}}{\rangeVal_{\numOfIterations + 1}} - \epsilon_{\numOfIterations + 1} \right)
\\ & \overset{\left( 1 \right)}{\le} \sum_{\rangeVal_{\numOfIterations + 1} \in \rangeSet_{\epsilon_{\numOfIterations + 1}}^{\query_{\numOfIterations + 1}}} \distPFunc{\viewVal_{\numOfIterations}}{\rangeVal_{\numOfIterations + 1}} \left( \func{\ell_{\distPDomainOfSets{\viewVal_{\numOfIterations}}}^{\query_{\numOfIterations + 1}}}{\rangeVal_{\numOfIterations + 1}} - \epsilon_{\numOfIterations + 1} \right)
\\ & = \distPFunc{\viewVal_{\numOfIterations}}{\rangeSet_{\epsilon_{\numOfIterations + 1}}^{\query_{\numOfIterations + 1}}} \left( \func{\ell_{\distPDomainOfSets{\viewVal_{\numOfIterations}}}^{\query_{\numOfIterations + 1}}}{\rangeSet_{\epsilon_{\numOfIterations + 1}}^{\query_{\numOfIterations + 1}}} - \epsilon_{\numOfIterations + 1} \right)
\\ & \overset{\left( 2 \right)}{\le} \delta_{\numOfIterations + 1},
\end{align*}
where in both cases (1) is a result of the fact that in both sums we add positive summands and remove negative ones, and (2) results from the inductive assumption.

Combining the two we get that $\distFunc{\viewSet_{\epsilon_{\left[ \numOfIterations + 1 \right]}}^{\numOfIterations + 1}} \cdot \left( \func{\ell_{\distDomainOfSets}^{\analyst}}{\viewSet_{\epsilon_{\left[ \numOfIterations + 1 \right]}}^{\numOfIterations + 1}} - \epsilon_{\left[ \numOfIterations + 1 \right]} \right) \le \sum_{i \in \left[ \numOfIterations + 1\right]} \delta_{i}$.
\end{proof}

\begin{lemma} [Azuma inequality extended to high probability bound] \label{lem:extendedAzuma}

Given $\numOfIterations \in \mathbb{N}$, $0 \le \epsilon_{1}, \ldots, \epsilon_{\numOfIterations}$, $0 \le \delta_{1}, \ldots, \delta_{\numOfIterations} \le 1$, if $Y_{0}, \ldots, Y_{\numOfIterations}$ is a martingale with respect to another sequence $Z_{0}, \ldots, Z_{\numOfIterations}$ such that for any $i \in \left[ \numOfIterations \right]$, $\prob{}{\left| Y_{i} - Y_{i - 1} \right| > \epsilon_{i}} \le \delta_{i}$, then for any $\lambda > 0$,
\[
\prob{}{\left| Y_{\numOfIterations} - Y_{0} \right| > \lambda } \le \func{\exp}{- \frac{\lambda^{2}}{2 \sum_{i = 1}^{\numOfIterations} \epsilon_{i}^{2}}} + \sum_{i = 1}^{\numOfIterations} \delta_{i}.
\]
\end{lemma}

The proof parallels that of a similar lemma by \cite{TV15} (their Proposition 34).
\begin{proof}
For any given realization of the random variable $y = \left(y_{0}, \ldots, y_{\numOfIterations} \right)$, we denote by $\func{I}{y}$ the first index $i$ for which $\left| y_{i} - y_{i - 1} \right| > \epsilon_{i}$. If no such index exists, $\func{I}{y} = \numOfIterations + 1$. We then define $\bar{y'} \coloneqq \func{f}{y}$ where $\forall i < \func{I}{\bar{y}} : y'_{i} = y_{i}$ and $\forall i \ge \func{I}{\bar{y}} : y'_{i} = y_{\func{I}{\bar{y}} - 1}$. Notice that the random variable $Y'$ is also a martingale with respect to $Z_{0}$, $\prob{}{\left| Y'_{i} - Y'_{i - 1} \right| > \epsilon_{i}} = 0$, and
\[
\prob{}{Y' \ne Y} \le \sum_{i = 1}^{\numOfIterations} \prob{}{Y'_{i} \ne Y_{i}} \le \sum_{i = 1}^{\numOfIterations} \delta_{i}.
\]

Using these facts we get
\begin{align*}
\prob{}{\left| Y_{\numOfIterations} - Y_{0} \right| > \lambda } & = \overbrace{\prob{}{Y' = Y}}^{\le 1} \cdot \prob{}{\left| Y'_{\numOfIterations} - Y'_{0} \right| > \lambda } + \prob{}{Y' \ne Y} \cdot \overbrace{\prob{}{\left| Y_{\numOfIterations} - Y_{0} \right| > \lambda }}^{\le 1}
\\ & \overset{\left( 1 \right)}{\le} \func{\exp}{- \frac{\lambda^{2}}{2 \sum_{i = 1}^{\numOfIterations} \epsilon_{i}^{2}}} + \sum_{i = 1}^{\numOfIterations} \delta_{i}.
\end{align*}
where (1) results from the previous inequality and Azuma's inequality for $Y'$.
\end{proof}

\begin{proof} [Missing parts from the proof of Theorem  \ref{thm:advancedAdaptiveComposition}]
\begin{align*}
\expectation{\rangeRV_{i + 1}}{Y_{i + 1} \,|\, \sampleSetRV, \rangeRV_{1}, \ldots, \rangeRV_{i}} & = \expectation{\rangeRV_{i + 1}}{\sum_{j = 1}^{i + 1} \left( \func{\ell_{\distPUpInd{}{\viewRV_{j - 1}}}}{\rangeRV_{j}} - \expectation{\rangeRV \sim \distPRange{\viewRV_{j - 1}}}{\func{\ell_{\distPUpInd{}{\viewRV_{j - 1}}}}{\rangeRV}} \right) \,|\, \sampleSetRV, \rangeRV_{1}, \ldots, \rangeRV_{i}}
\\ & = \overbrace{\sum_{j = 1}^{i} \left( \func{\ell_{\distPUpInd{}{\viewRV_{j - 1}}}}{\rangeRV_{j}} - \expectation{\rangeRV \sim \distPRange{\viewRV_{j - 1}}}{\func{\ell_{\distPUpInd{}{\viewRV_{j - 1}}}}{\rangeRV}} \right)}^{ = \func{Y_{i}}{\sampleSetRV, \rangeRV_{1}, \ldots, \rangeRV_{i}}}
\\ &~~~ + \overbrace{\expectation{\rangeRV_{i + 1}}{ \func{\ell_{\distPUpInd{}{\viewRV_{i}}}}{\rangeRV_{i + 1}} - \expectation{\rangeRV \sim \distPRange{\viewRV_{i}}}{\func{\ell_{\distPUpInd{}{\viewRV_{i}}}}{\rangeRV}} \,|\, \sampleSetRV, \rangeRV_{1}, \ldots, \rangeRV_{i}}}^{= 0 }
\\ & = \func{Y_{i}}{\sampleSetRV, \rangeRV_{1}, \ldots, \rangeRV_{i}}
\end{align*}
where the expectation is taken over the random process, which has randomness that results from the choice of $\sampleSet \in \domainOfSets$ and the internal probability of $\mechanism$.

\begin{align*}
\underset{\viewRV \sim \distInd{\view_{\numOfIterations}}}{\text{Pr}} & \left[ \func{\ell_{\distDomainOfSets}}{\viewRV} > \epsilon' \right]
\\ & \overset{\left( 1 \right)}{\le} \prob{\viewRV \sim \distInd{\view_{\numOfIterations}}}{\sum_{i = 1}^{\numOfIterations} \func{\ell_{\distPUpInd{}{\viewRV_{i - 1}}}}{\rangeRV_{i}} > \sqrt{8 \func{\ln}{\frac{1}{\delta'}} \sum_{i = 1}^{\numOfIterations} \epsilon_{i}^{2}} + \sum_{i = 1}^{\numOfIterations} \alpha_{i}}
\\ & \overset{\left( 2 \right)}{\le}  \prob{\viewRV \sim \distInd{\view_{\numOfIterations}}}{\sum_{j = 1}^{\numOfIterations} \left( \func{\ell_{\distPUpInd{}{\viewRV_{j - 1}}}}{\rangeRV_{j}} - \expectation{\rangeRV \sim \distPRange{\viewRV_{j - 1}}}{\func{\ell_{\distPUpInd{}{\viewRV_{j - 1}}}}{\rangeRV}} \right) > \sqrt{8 \func{\ln}{\frac{1}{\delta'}} \sum_{i = 1}^{\numOfIterations} \epsilon_{i}^{2}}}
\\ & \overset{\left( 3 \right)}{=} \prob{\viewRV \sim \distInd{\view_{\numOfIterations}}}{Y_{\numOfIterations} - \overbrace{Y_{0}}^{= 0} > \sqrt{8 \func{\ln}{\frac{1}{\delta'}} \sum_{i = 1}^{\numOfIterations} \epsilon_{i}^{2}}}
\\ & \overset{\left( 4 \right)}{\le} \delta' + \sum_{i = 1}^{\numOfIterations} \frac{\delta_{i}}{\epsilon_{i}}
\end{align*}
where (1) results from Lemma \ref{lem:liniarityOfLoss}, (2) from the bound on the expectation of the stability loss, (3) from the definition of $Y_{i}$, and (4) from Lemma \ref{lem:extendedAzuma}.
\end{proof}
\section{Missing Details from Section~\ref{Generalization}} \label{sec:genProofs}

\begin{proof}[Missing parts from the proof of Theorem \ref{thm:genOfExpectation}]
\begin{align*}
\expectation{\sampleSetRV \sim \distDomainOfSets, \queryRV' \sim \mechanismFunc{\sampleSetRV, \query}}{\func{\queryRV'}{\distDomainOfSets}} & = \sum_{\sampleSet \in \domainOfSets} \distFunc{\sampleSet} \cdot \sum_{\query' \in \queriesFamily_{\Delta}} \distFuncDep{\query'}{\sampleSet} \cdot \func{\query'}{\distDomainOfSets}
\\ & = \sum_{\query' \in \queriesFamily_{\Delta}} \overbrace{\sum_{\sampleSet \in \domainOfSets} \distFunc{\sampleSet} \cdot \distFuncDep{\query'}{\sampleSet}}^{=\distFunc{\query'}} \cdot \overbrace{\sum_{\sampleSet' \in \domainOfSets} \func{\distDomainOfSets}{\sampleSet'} \cdot \func{\query}{\sampleSet'}}^{=\func{\query'}{\distDomainOfSets}}
\\ & = \sum_{\query' \in \queriesFamily_{\Delta}} \distFunc{\query'} \sum_{\domainVal \in \domain} \overbrace{\sum_{\sampleSet' \in \domainOfSets} \distFunc{\sampleSet'} \cdot \distFuncDep{\domainVal}{\sampleSet'}}^{\distFunc{\domainVal}} \cdot \func{\query'}{\domainVal}
\\ & = \sum_{\query' \in \queriesFamily_{\Delta}} \distFunc{\query'} \sum_{\domainVal \in \domain} \distFunc{\domainVal} \cdot \func{\query'}{\domainVal}
\end{align*}

\begin{align*}
\expectation{\sampleSetRV \sim \distDomainOfSets, \queryRV' \sim \mechanismFunc{\sampleSetRV, \query}}{\func{\queryRV'}{\sampleSetRV}} & = \sum_{\sampleSet \in \domainOfSets} \distFunc{\sampleSet} \cdot \sum_{\query' \in \queriesFamily_{\Delta}} \distFuncDep{\query'}{\sampleSet} \cdot \func{\query'}{\sampleSet}
\\ & = \sum_{\query' \in \queriesFamily_{\Delta}} \sum_{\domainVal \in \domain} \overbrace{\sum_{\sampleSet \in \domainOfSets} \distFunc{\sampleSet} \cdot \distFuncDep{\query'}{\sampleSet} \cdot \distFuncDep{\domainVal}{\sampleSet}}^{=\func{\distJoint{\domain}{\queriesFamily_{\Delta}}{\query}}{\domainVal, \query'}} \cdot \func{\query'}{\domainVal}
\\ & \overset{\left( 1 \right)}{=} \sum_{\query' \in \queriesFamily_{\Delta}} \distFunc{\query'} \sum_{\domainVal \in \domain} \distFuncDep{\domainVal}{\query'} \cdot \func{\query'}{\domainVal},
\end{align*}
where (1) is a result of Lemma \ref{lem:depDist}.
\end{proof}

\begin{proof}[Proof of Claim \ref{clm:monLSS}]
Since $\query^{i}$ is a post-processing of $\viewVal^{i}$ and $\text{Adp}_{\bar{\mechanism}}$ is $\tuple{\epsilon}{\delta}$-LSS with respect to $\analyst$, Theorem \ref{thm:postprocessing} implies that the post-processing producing $\query^{i}$ is $\tuple{\epsilon}{\delta}$-LSS with respect to $\analyst$ as well. 
Using Lemma \ref{lem:notToManyBad} we get that $\distFunc{\querySet_{2 \epsilon}} < \frac{\delta}{\epsilon}$ for each of the $\numOfMetaIterations$ rounds.
Using the union bound and the fact that the $\numOfMetaIterations$ rounds are independent we get that $\prob{\bar{\sampleSetRV} \sim \distDomainOfSets^{\numOfMetaIterations}, \left( \queryRV, \rangeRV, I \right) \sim \func{\text{Mon}_{\bar{\mechanism}}}{\bar{\sampleSetRV}, \analyst}}{\query \in \querySet_{2 \epsilon}} < \frac{\numOfMetaIterations \delta}{\epsilon}$.
This allows us to invoke Theorem \ref{thm:genOfExpectation}, with $\frac{\numOfMetaIterations \delta}{\epsilon}$ replacing $\delta$.\footnote{The fact that repeating this process $\numOfMetaIterations$ independent times affects only the $\delta$ and not the $\epsilon$ will be crucial to the move from generalization of expectation to generalization with high probability (at least in this proof technique). This is made possible by the way $\rangeVal$'s were treated in the distance measure in the LSS definition. For comparison, see the remark in Lemma 3.3 in \cite{BNSSSU16}. We hypothesize, quite informally, that stability definitions that degrade in the $\epsilon$ term on multiple independent runs cannot yield generalization with high probability. As far as we are aware, all previously studied stability notions support this claim.}
\end{proof}

\begin{proof}[Proof of Claim \ref{clm:Monsampleerror}]
This is a direct result of combining the sample accuracy definition and the union bound. If the probability that the sample accuracy of $\mechanism$ will be greater than $\epsilon$ is bounded by $\delta$, then the probability that it will fail to hold once in $\numOfMetaIterations$ independent iterations is less then $\numOfMetaIterations \delta$, and since the values of the query are bounded on the interval $\left[-\Delta, \Delta \right]$ the maximal error in these cases is $2 \Delta$.
\end{proof}

\begin{proof}[Proof of Claim \ref{clm:Mondisterror}]
First recall that from the definition of the monitor mechanism, $\forall i \in \left[\numOfMetaIterations \right], \func{\query_{i}}{\distDomainOfSets} - \rangeVal_{i} \ge 0$. Therefore if $\mechanism$ is not $\tuple{\epsilon}{\delta}$-Distribution Accurate, then $\forall i \in \left[\numOfMetaIterations \right]$
\[
\prob{\sampleSetRV \sim \distDomainOfSets^{\numOfMetaIterations}, \viewRV \sim \func{\mechanism_{i}}{\sampleSetRV, \analyst}, \left( \queryRV, \rangeRV \right) = \underset{\tuple{\query}{\rangeVal} \in \viewRV}{\arg\max} \left|\func{\query}{\distDomainOfSets} - \rangeVal \right|}{\func{\queryRV}{\distDomainOfSets} - \rangeRV > \epsilon} > \delta.
\]
Since the $\numOfMetaIterations$ rounds of the monitor mechanism are independent and $i^{*}$ is the index of the round with the maximal error,
\[
\prob{\bar{\sampleSetRV} \sim \distDomainOfSets^{\numOfMetaIterations}, \left( \queryRV, \rangeRV, I \right) \sim \func{\text{Mon}_{\bar{\mechanism}}}{\bar{\sampleSetRV}, \analyst}}{\func{\queryRV}{\distDomainOfSets} - \rangeRV > \epsilon} > 1 - \left(1 - \delta \right)^{\numOfMetaIterations}.
\]
So the expectation of this quantity must be greater then $\epsilon \left(1 - \left(1 - \delta \right)^{\numOfMetaIterations} \right)$, concluding the proof.
\end{proof}
\section{Missing Details from Section \ref{sec:relToNotions}}

\subsection{Proofs of Implication Theorems} \label{sec:implicationsProofs}

\begin{proof}[Proof of Theorem \ref{thm:DPimpLMI}]
Given $\mathbf{b} \subseteq \domain \times \range$ we denote $\func{\rangeSet_{\mathbf{b}}}{\domainVal} \coloneqq \left\{\rangeVal \in \range \,|\, \tuple{\domainVal}{\rangeVal} \in \mathbf{b} \right\}$ (which might be empty for some $\domainVal$'s). Using this notation we prove that for any $\mathbf{b} \subseteq \domain \times \range$,
\begin{align*}
\func{\distJoint{\domain}{\range}{}}{\mathbf{b}} & =  \sum_{\domainVal \in \domain} \distFunc{\domainVal} \distFuncDep{\func{\rangeSet_{\mathbf{b}}}{\domainVal}}{\domainVal}
\\ & \overset{\left( 1 \right)}{=} \overbrace{\sum_{\domainVal' \in \domain} \distFunc{\domainVal'}}^{=1} \sum_{\domainVal \in \domain} \distFunc{\domainVal} \sum_{\sampleSet' \in \domain^{\sampleSize - 1}} \distFunc{\sampleSet'} \cdot \distFuncDep{\func{\rangeSet_{\mathbf{b}}}{\domainVal}}{\sampleSet' \cup \left\{ \domainVal \right\}}
\\ & \overset{\left( 2 \right)}{\le}  \sum_{\domainVal \in \domain} \distFunc{\domainVal} \sum_{\domainVal' \in \domain} \distFunc{\domainVal'} \sum_{\sampleSet' \in \domain^{\sampleSize - 1}} \distFunc{\sampleSet'} \left( e^{\epsilon} \cdot \distFuncDep{\func{\rangeSet_{\mathbf{b}}}{\domainVal}}{\sampleSet' \cup \left\{ \domainVal' \right\}} + \delta \right)
\\ & \overset{\left( 1 \right)}{=} \sum_{\domainVal \in \domain} \distFunc{\domainVal} \sum_{\sampleSet \in \domainOfSets} \distFunc{\sampleSet} \left( e^{\epsilon} \cdot \distFuncDep{\func{\rangeSet_{\mathbf{b}}}{\domainVal}}{\sampleSet} + \delta \right)
\\ & =  \sum_{\domainVal \in \domain} \distFunc{\domainVal} \left( e^{\epsilon} \cdot \distFunc{\func{\rangeSet_{\mathbf{b}}}{\domainVal}} + \delta \right)
\\ & =  e^{\epsilon} \cdot \func{\distProd{\domain}{\range}{}}{\mathbf{b}} + \delta,
\end{align*}
where (1) are a result of the fact that $\distDomainOfSets$ is a product distribution, and (2) is a result of the DP definition.
The proof is concluded by repeating the same process for the second direction.
\end{proof}

\begin{proof}[Proof of Theorem \ref{thm:MIimpLMI}]
Notice that the proof of DP holding under post-processing (see e.g. \cite{DR14}), proves in fact that $\tuple{\epsilon}{\delta}$-indistinguishability is closed under post-processing. Since $\domainVal$ is a post-processing of $\sampleSet$, the fact that $\distJoint{\domainOfSets}{\range}{}$ and $\distProd{\domainOfSets}{\range}{}$ are $\tuple{\epsilon}{\delta}$ indistinguishable implies that $\distJoint{\domain}{\range}{}$ and $\distProd{\domain}{\range}{}$ are indistinguishable as well.
\end{proof}

\begin{proof}[Proof of Theorem \ref{thm:TSimpLMI}]
Given any $\mathbf{b} \subseteq \domain \times \range$ we denote $\func{\rangeSet_{\mathbf{b}}}{\domainVal} \coloneqq \left\{\rangeVal \in \range \,|\, \tuple{\domainVal}{\rangeVal} \in \mathbf{b} \right\}$ (which might be empty for some $\domainVal$'s). Using this notation and the subset $\sampleSetSet$ from Definition \ref{def:typStab} we prove that for any $\mathbf{b} \subseteq \domain \times \range$,
\begin{align*}
\func{\distJoint{\domain}{\range}{}}{\mathbf{b}} & = \sum_{\domainVal \in \domain} \distFunc{\domainVal} \distFuncDep{\func{\rangeSet_{\mathbf{b}}}{\domainVal}}{\domainVal}
\\ & \overset{\left( 1 \right)}{=} \sum_{\domainVal \in \domain} \sum_{\sampleSet \in \domainOfSets} \distFunc{\sampleSet} \distFuncDep{\domainVal}{\sampleSet} \overbrace{\sum_{\sampleSet' \in \domainOfSets} \distFunc{\sampleSet'}}^{= 1} \distFuncDep{\func{\rangeSet_{\mathbf{b}}}{\domainVal}}{\sampleSet}
\\ & \overset{\left( 2 \right)}{\le} e^{\epsilon} \sum_{\domainVal \in \domain} \sum_{\sampleSet \in \sampleSetSet} \distFunc{\sampleSet} \distFuncDep{\domainVal}{\sampleSet} \sum_{\sampleSet' \in \sampleSetSet} \distFunc{\sampleSet'} \distFuncDep{\func{\rangeSet_{\mathbf{b}}}{\domainVal}}{\sampleSet'} + \delta + 2 \eta
\\ & \le e^{\epsilon} \sum_{\domainVal \in \domain}\overbrace{\sum_{\sampleSet \in \domainOfSets} \distFunc{\sampleSet} \distFuncDep{\domainVal}{\sampleSet}}^{= \distFunc{\domainVal}} \overbrace{\sum_{\sampleSet' \in \domainOfSets} \distFunc{\sampleSet'} \distFuncDep{\func{\rangeSet_{\mathbf{b}}}{\domainVal}}{\sampleSet'}}^{= \distFunc{\func{\rangeSet_{\mathbf{b}}}{\domainVal}}} + \delta + 2 \eta
\\ & = e^{\epsilon} \func{\distProd{\domain}{\range}{}}{\mathbf{b}} + \delta + 2 \eta
\end{align*}
where (1) results from the fact that $\domainVal$ and $\rangeVal$ are independent given $\sampleSet$, and (2) from the definition of TS.
\end{proof}

\begin{proof} [Proof of Theorem \ref{thm:LMLimpLMI}]
The proof is identical to the one used by \cite{EGI19} when proving that ML implies MI (Theorem 7).
\end{proof}

\begin{proof} [Proof of Theorem \ref{thm:LMIimpLSS}]
Assume $\mechanism$ is not $\tuple{\epsilon'}{\frac{\delta}{\epsilon}}$-LSS, which means that in particular $\distFunc{\rangeSet_{\epsilon'}} > \frac{\delta}{\epsilon}$. Denoting $B_{\domain \times \range}^{\epsilon'} \coloneqq \underset{\rangeVal \in \rangeSet_{\epsilon'}}{\cup} \left( \func{\domainSet_{+}}{\rangeVal} \times \left\{ \rangeVal \right\} \right)$ we get that from the definition of the stability loss,
\[
\func{\distJoint{\domain}{\range}{}}{B_{\domain \times \range}^{\epsilon'}} - \func{\distProd{\domain}{\range}{}}{B_{\domain \times \range}^{\epsilon'}} = \sum_{\rangeVal \in \rangeSet_{\epsilon'}} \distFunc{\rangeVal} \cdot \func{\ell}{\rangeVal} > \epsilon' \cdot \distFunc{\rangeSet_{\epsilon'}}.
\]

But on the other hand, from the fact that $\mechanism$ is $\tuple{\epsilon}{\delta}$-LMI we get in contradiction that
\begin{align*}
\func{\distJoint{\domain}{\range}{}}{B_{\domain \times \range}^{\epsilon'\epsilon'}} - \func{\distProd{\domain}{\range}{}}{B_{\domain \times \range}^{\epsilon'}} & \le  \func{\distProd{\domain}{\range}{}}{B_{\domain \times \range}^{\epsilon'}} \cdot \left( e^\epsilon - 1 \right) + \delta 
\\ & \overset{\left( 1 \right)}{\le} \distFunc{\rangeSet_{\epsilon'}} \cdot \left( e^\epsilon - 1 \right) + \epsilon \cdot \distFunc{\rangeSet_{\epsilon'}}
\\ & \overset{\left( 2 \right)}{\le} \epsilon' \cdot \distFunc{\rangeSet_{\epsilon'}}
\end{align*}
where (1) results from the fact that $\epsilon \cdot \distFunc{\rangeSet_{\epsilon'}} > \delta$, and (2) from the definition of $\epsilon'$ and the assumption that $\mechanism$ is not $\tuple{\epsilon'}{\frac{\delta}{\epsilon}}$-LSS. The proof is concluded by repeating the same process for the second direction.
\end{proof}

\begin{lemma} [see, e.g., \cite{SSBD14} Theorem 30.2] \label{lem:CSnoOverFit}
Given $0 \le \delta \le 1$, $m \le \frac{\sampleSize}{2}$, a domain $\domain$, and a distribution $\distDomain$ defined over it, we denote by $\mathcal{H}$ the family of functions (usually referred to as \emph{hypothesis} in the context of Machine Learning) of the form $h : \domain \rightarrow \left\{ 0, 1 \right\}$, and let $h^{*} \in \mathcal{H}$ be some unique hypothesis which we will think of as the \emph{true hypothesis}. We will refer to $\func{h^{*}}{\domainVal}$ as the true \emph{label} of $\domainVal$, and denote the labeled domain by $\domain_{h^{*}} \coloneqq \left\{ {\tuple{\domainVal}{\func{h^{*}}{\domainVal}} \,|\, \domainVal \in \domain} \right\}$. Let $\mechanism : \domainOfSets \times \queriesFamily \rightarrow \mathcal{H}$ be a mechanism with a compression scheme (Definition \ref{def:ComprSchm}), In this case, with probability (over the sampling of $\sampleSet$ and the internal randomness of the mechanism in case it is non deterministic) greater then $1 - \delta$ we have that,
\[
\left| \func{h_{w}}{\sampleSet \setminus w} - \func{h_{w}}{\distDomain} \right| \le \sqrt{\func{h_{w}}{\sampleSet \setminus w} \frac{4 m \func{\ln}{2 \sampleSize / \delta}}{\sampleSize}} + \frac{8 m \func{\ln}{2 \sampleSize / \delta}}{\sampleSize}
\]
where $\func{h_{w}}{\sampleSet \setminus w}$ is the empirical mean of $h_{w}$ over $\sampleSet \setminus w$ and $\func{h_{w}}{\distDomain}$ is its expectation with respect to $\distDomain$.
\end{lemma}

\begin{proof} [Proof of Theorem \ref{thm:CSimpLSS}]
We will prove that $g$ is $\tuple{\epsilon}{\delta}$-LSS for such an $\epsilon$, and since LSS holds under post-processing, this suffices. Notice that now $\range = \domain^{m}$. This proof resembles that of \cite{CLNRW16}.

We start by analyzing the loss of $w$ and get that,
\begin{align*}
\func{\ell}{w} & = \sum_{\domainVal \in \func{\domainSet_{+}}{w}} \left( \distFuncDep{\domainVal}{w} - \distFunc{\domainVal}\right)
\\ & = \sum_{\domainVal \in \func{\domainSet_{+}}{w}} \sum_{\sampleSet \in \domainOfSets} \distFuncDep{\sampleSet}{w} \left( \distFuncDep{\domainVal}{\sampleSet} - \distFunc{\domainVal}\right)
\\ & = \sum_{\sampleSet \in \domainOfSets} \distFuncDep{\sampleSet}{w} \sum_{\domainVal \in \func{\domainSet_{+}}{w}} \left( \frac{m}{\sampleSize} \distFuncDep{\domainVal}{w} + \frac{\sampleSize - m}{\sampleSize} \distFuncDep{\domainVal}{\sampleSet \setminus w} - \distFunc{\domainVal}\right)
\\ & \le \sum_{\sampleSet \in \domainOfSets} \distFuncDep{\sampleSet}{w} \left( \frac{m}{\sampleSize} + \sum_{\domainVal \in \func{\domainSet_{+}}{w}} \left( \distFuncDep{\domainVal}{\sampleSet \setminus w} - \distFunc{\domainVal}\right) \right)
\\ & = \sum_{\sampleSet \in \domainOfSets} \distFuncDep{\sampleSet}{w} \left( \frac{m}{\sampleSize} + \sum_{\domainVal \in \domain} \left( \distFuncDep{\domainVal}{\sampleSet \setminus w} - \distFunc{\domainVal}\right) \func{h_{w}^{+}}{\domainVal} \right)
\\ & = \sum_{\sampleSet \in \domainOfSets} \distFuncDep{\sampleSet}{w} \left( \frac{m}{\sampleSize} + \func{h_{w}^{+}}{\sampleSet \setminus w} - \func{h_{w}^{+}}{\distDomain} \right)
\end{align*}
where $\func{h_{w}^{+}}{\domainVal}$ is simply the characteristic function of $\func{\domainSet_{+}}{w}$.

Using this inequality we get that $\forall \rangeSet \subseteq \range$,
\begin{align*}
\distFunc{\rangeSet} \left( \func{\ell}{\rangeSet} - \epsilon \right) & = \sum_{w \in \rangeSet} \distFunc{w} \left( \func{\ell}{w} - \epsilon \right)
\\ & \overset{\left( 1 \right)}{\le} \sum_{w \in \rangeSet} \distFunc{w} \sum_{\sampleSet \in \domainOfSets} \distFuncDep{\sampleSet}{w} \left( \frac{m}{\sampleSize} + \func{h_{w}^{+}}{\sampleSet \setminus w} - \func{h_{w}^{+}}{\distDomain} - \epsilon \right)
\\ & = \sum_{\sampleSet \in \domainOfSets} \distFunc{\sampleSet} \sum_{w \in \rangeSet} \distFuncDep{w}{\sampleSet} \left( \func{h_{w}^{+}}{\sampleSet \setminus w} - \func{h_{w}^{+}}{\distDomain} - \left( \epsilon - \frac{m}{\sampleSize} \right) \right)
\\ & \le \sum_{\sampleSet \in \domainOfSets} \distFunc{\sampleSet} \max_{w = \func{g}{\sampleSet}, h = \func{f}{w}} \left( \func{h}{\sampleSet \setminus w} - \func{h}{\distDomain} - \left( \epsilon - \frac{m}{\sampleSize} \right) \right)
\\ & \overset{\left( 2 \right)}{\le} \prob{\sampleSetRV \sim \distDomainOfSets, W \sim \func{g}{\sampleSetRV}, H \sim \func{f}{W}}{\func{H}{\sampleSetRV \setminus W} - \func{H}{\distDomain} > \left( \epsilon - \frac{m}{\sampleSize} \right)}
\\ & \overset{\left( 3 \right)}{\le} \sqrt{\frac{4 m \func{\ln}{2 \sampleSize / \delta}}{\sampleSize}} + \frac{8 m \func{\ln}{2 \sampleSize / \delta}}{\sampleSize} + \frac{m}{\sampleSize}
\\ & \overset{\left( 4 \right)}{\le} 11 \sqrt{\frac{m \func{\ln}{2 \sampleSize / \delta}}{\sampleSize}}
\end{align*}
where (1) results from the previous inequality, (2) from the fact that we removed $\sampleSet$'s for which the summand is negative, and replaced the positive ones with 1 - which is greater then the maximal possible value, (3) from Lemma \ref{lem:CSnoOverFit} and the fact that the value of $h$ is bounded by 1, and (4) from the fact that $m \le \frac{\sampleSize}{9 \func{\ln}{\frac{2 \sampleSize}{\delta}}}$.
\end{proof}

\subsection{Proofs of Separation Theorems} \label{sec:separationsProofs}

\begin{proof} [Proof of Theorem \ref{thm:MIgrLMI}]
Without loss of generality, assume $0 < \epsilon \le 0.7$. Given $0 \le \alpha \le \frac{\epsilon}{7}$, $p = \frac{1}{2} + \alpha$ we will define some function $f : \domain \rightarrow \left\{ 0, 1\right\}$, and for $i \in \left\{ 0, 1\right\}$ denote $\domainSet_{i} \coloneqq \left\{ \domainVal \in \domain \,|\, \func{f}{\domainVal} = i \right\}$, set an arbitrary distribution $\distDomain$ such that $\distFunc{\domainSet_{1}} = p$, and $\distDomainOfSets$ which is the product of $\distDomain$. We will consider a mechanism $\mechanism$ which in response to a query $\query$ returns the parity function of the vector $\left( \func{f}{\sampleSet_{1}}, \ldots, \func{f}{\sampleSet_{\sampleSize}} \right)$, where $\sampleSet_{1}, \ldots \sampleSet_{\sampleSize}$ denotes the elements of the sample set $\sampleSet$. Formally, $\func{\mechanism}{\query, \sampleSet} = \left| \sampleSet \cap \domainSet_{1} \right| \pmod{2}$, and we prove that this mechanism is $\tuple{\epsilon}{0}$-LMI but not $\tuple{1}{\frac{1}{5}}$-MI.

We start with denoting by $p_{\sampleSize - 2}$ the probability that the parity function of a sample of size $\sampleSize - 2$ will be equal to $1$, and the possible outputs as $r_{0}, r_{1}$. Notice that,

\[
\distFuncDep{r_{1}}{\domainSet_{1}} = p \cdot p_{\sampleSize - 2} + \left( 1 - p \right) \left( 1 - p_{\sampleSize - 2} \right)
\]
\[
\distFuncDep{r_{1}}{\domainSet_{0}} = \left( 1 - p \right) p_{\sampleSize - 2} + p \left( 1 - p_{\sampleSize - 2} \right) = 1 - \distFuncDep{r_{1}}{\domainSet_{1}}
\]
\begin{align*}
\distFunc{r_{1}} & = p \cdot \distFuncDep{r_{1}}{\domainSet_{1}} +  \left( 1 - p \right) \distFuncDep{r_{1}}{\domainSet_{0}}
\\ & = \left( 2 p - 1 \right) \distFuncDep{r_{1}}{\domainSet_{1}} + 1 - p
\\ & = \left( 1 - 2 p \right) \distFuncDep{r_{1}}{\domainSet_{0}} + p
\end{align*}

Using these identities we will first prove that $\frac{\distFunc{r_{1}}}{\distFuncDep{r_{1}}{\domainSet_{1}}}, \frac{\distFunc{r_{1}}}{\distFuncDep{r_{1}}{\domainSet_{0}}} \le e^{\epsilon}$. Since a similar claim can be proven for $\frac{\distFuncDep{r_{1}}{\domainSet_{1}}}{\distFunc{r_{1}}}, \frac{\distFuncDep{r_{1}}{\domainSet_{0}}}{\distFunc{r_{1}}}$, we get that this mechanism is $\tuple{\epsilon}{0}$-LMI.

\begin{align*}
\frac{\distFunc{r_{1}}}{\distFuncDep{r_{1}}{\domainSet_{1}}} & = \frac{\left( 2 p - 1 \right) \distFuncDep{r_{1}}{\domainSet_{1}} + 1 - p}{\distFuncDep{r_{1}}{\domainSet_{1}}}
\\ & = 2 p - 1 + \frac{1 - p}{\left( 2 p - 1 \right) p_{\sampleSize - 2} + 1 - p}
\\ & = 2 p - \frac{\left( 2 p - 1 \right) p_{\sampleSize - 2}}{\left( 2 p - 1 \right) p_{\sampleSize - 2} + 1 - p}
\\ & = 1 + 2 \alpha - \overbrace{\frac{ 2 \alpha p_{\sampleSize - 2}}{\alpha \left( 2 p_{\sampleSize - 2} - 1 \right) + \frac{1}{2}}}^{\ge 0}
\\ & \overset{\left( 1 \right)}{\le} 1 + \overbrace{2 \alpha}^{\le \epsilon}
\\ & \overset{\left( 2 \right)}{\le} e^{\epsilon}
\end{align*}
where (1) results from the fact that $0 \le \alpha < \frac{\epsilon}{7} \le \frac{1}{10} $, so the denominator $\alpha \left( 2 p_{\sampleSize - 2} - 1 \right) + \frac{1}{2}$ must be positive, and (2) is a result of the inequality $1 + \epsilon \le e^{\epsilon}$ for any $\epsilon < 1$. Similarly we get that,

\begin{align*}
\frac{\distFunc{r_{1}}}{\distFuncDep{r_{1}}{\domainSet_{0}}} & = \frac{\left( 1 - 2 p \right) \distFuncDep{r_{1}}{\domainSet_{0}} + p}{\distFuncDep{r_{1}}{\domainSet_{0}}}
\\ & = 1 - 2 p + \frac{p}{\distFuncDep{r_{1}}{\domainSet_{0}}}
\\ & = 2 - 2 p - \frac{\left( 1 - 2 p \right) p_{\sampleSize - 2}}{\left( 1 - 2 p \right) p_{\sampleSize - 2} + p}
\\ & = 1 + 2 \alpha + \overbrace{\frac{ 2 \alpha \cdot p_{\sampleSize - 2}}{\alpha \left( 1 - 2 p_{\sampleSize - 2} \right) + \frac{1}{2}}}^{\le 5 \alpha}
\\ & \overset{\left( 1 \right)}{\le} 1 + \overbrace{7 \alpha}^{\le \epsilon}
\\ & \overset{\left( 2 \right)}{\le} e^{\epsilon}
\end{align*}
where (1) results from the fact that $0 \le \alpha < \frac{\epsilon}{7} \le \frac{1}{10}$, and $0 \le p_{\sampleSize - 2} \le 1$, so $\alpha \left( 1 - 2 p_{\sampleSize - 2} \right) + \frac{1}{2} \ge \frac{4}{10}$, and (2) is a result of the inequality $1 + \epsilon \le e^{\epsilon}$ for any $\epsilon < 1$.

On the other hand, we will prove the response  dramatically changes the distribution over the sample sets. Using the fact that the parity function of a Binomial random variable $\func{\mathbf{b}}{n, p}$ is a Bernoulli random variable $\func{\text{Ber}}{\frac{1 - \left( 1 - 2 p \right)^n}{2}}$, and denoting $\mathcal{\sampleSetRV}_{1}$ the set of all sample sets with parity value $1$, we get that,
\begin{align*}
\func{\distProd{\domainOfSets}{\range}{}}{\mathcal{\sampleSetRV}_{1} \times \left\{ \rangeVal_{0} \right\}} & = \overbrace{\distFunc{\mathcal{\sampleSetRV}_{1}}}^{\distFunc{\rangeVal_{1}}} \cdot \distFunc{\rangeVal_{0}}
\\ & = \frac{1 - \left( 1 - 2 p \right)^{2 n}}{4} \\ & = e^{1} \overbrace{\func{\distJoint{\domainOfSets}{\range}{}}{\mathcal{\sampleSetRV}_{1} \times \left\{ \rangeVal_{0} \right\}}}^{= 0} + \frac{1 - \left( 2 \alpha \right)^{2 n}}{4}
\\ & \overset{\left( 1 \right)}{>} e^{1} \func{\distJoint{\domainOfSets}{\range}{}}{\mathcal{\sampleSetRV}_{1} \times \left\{ \rangeVal_{0} \right\}} + \frac{1}{5}
\end{align*}
where (1) is a result of the fact that $0 \le \alpha < \frac{\epsilon}{7} \le \frac{1}{10}$, $\sampleSize \ge 3$ and $\frac{1 - \left(\frac{1}{5} \right)^6}{4} > \frac{1}{5}$, which means this mechanism is not $\tuple{1}{\frac{1}{5}}$-MI.
\end{proof}

\begin{proof} [Proof of Theorem \ref{thm:LMIgrLSS}]
Without loss of generality $0 \le \delta \le 0.1$, so $\sampleSize > 2 \func{\ln}{\frac{2}{\delta}}$. Given $N > \sampleSize^{2}$, $\domain \coloneqq \left[ N \right]$, an arbitrary $\distDomain$ such that $\forall \domainVal \in \domain : \func{\distDomain}{\domainVal} \le \frac{1}{\sampleSize^{2}}$, and $\distDomainOfSets$ which is the product of $\distDomain$, we consider a mechanism $\mechanism$ which in response to some query $\query$ uniformly samples one element from its sample set and outputs it.

The fact that this mechanism is $\tuple{11 \sqrt{\frac{\func{\ln}{2 \sampleSize / \delta}}{\sampleSize}}}{\delta}$-LSS is a direct result of Theorem \ref{thm:CSimpLSS} for $m = 1$. On the other hand, notice that any $\rangeVal \in \range$ encodes one sample element which we will denote by $\func{\domainVal}{\rangeVal}$. Using this notation we will define the set $\mathbf{b} \coloneqq \underset{\rangeVal \in \range}{\cup} \tuple{\func{\domainVal}{\rangeVal}}{\rangeVal}$.
\begin{align*}
\func{\distJoint{\domain}{\range}{}}{\mathbf{b}} & = \sum_{\rangeVal \in \range} \distFunc{\rangeVal} \cdot  \distFuncDep{\func{\domainVal}{\rangeVal}}{\rangeVal}
\\ & \overset{\left( 1 \right)}{\ge} \sum_{\rangeVal \in \range} \distFunc{\rangeVal} \cdot \frac{1}{\sampleSize}
\\ & \overset{\left( 2 \right)}{>} \sum_{\rangeVal \in \range} \distFunc{\rangeVal} e \frac{1}{\sampleSize^{2}} + \overbrace{\sum_{\rangeVal \in \range} \distFunc{\rangeVal}}^{= 1} \frac{1}{2 \sampleSize}
\\ & \ge e \sum_{\rangeVal \in \range} \distFunc{\rangeVal} \cdot  \overbrace{\distFunc{\func{\domainVal}{\rangeVal}}}^{\le \frac{1}{n^{2}}} + \frac{1}{2 \sampleSize}
\\ & = e^{1} \cdot \func{\distProd{\domain}{\range}{}}{\mathbf{b}} + \frac{1}{2 \sampleSize}
\end{align*}
where (1) is a result of the fact that if all elements in the sample set differ from each other, with probability $\frac{1}{\sampleSize}$ the sampling mechanism will return the same sample element which was encoded by $\rangeVal$ and if not then the probability is only higher, and (2) is a result of the definitions of $\delta$ and $\sampleSize$. This proves the mechanism is not $\tuple{1}{\frac{1}{2 \sampleSize}}$-LMI.
\end{proof}

\section{Missing Details from Section \ref{sec:appAndDisc}} \label{apd:appAndDisc}
Definitions and properties in this section are due to \cite{DR14}.
\begin{definition} [Laplace Mechanism]
Given $0 \le b$ and a query $\query \in \queriesFamily_{\Delta}$, the Laplace mechanism with parameter $b$ is defined as:
\[
\mechanismFunc{\sampleSet, \query} = \func{\query}{\sampleSet} + \text{Lap}_{b}
\]
where $\text{Lap}_{b}$ is a random variable with unbiased Laplace distribution, which if a symmetric exponential distribution. Formally:
\[
\func{\text{Lap}_{b}}{x} = \frac{1}{2 b} e^{- \frac{\left| x \right|}{b}}
\]
\end{definition}

\begin{theorem}[Laplace Mechanism is Differentially Private]
Given $0 \le b, \epsilon$ and a query $\query \in \queriesFamily_{\Delta}$, the Laplace mechanism with parameter $b$ is $\tuple{\frac{2 \Delta}{\sampleSize \cdot b}}{0}$-DP.
\end{theorem}

\begin{theorem}[Laplace Mechanism is Sample Accurate]
Given $0 \le b$, $0 < \delta \le 1$ and a query $\query \in \queriesFamily_{\Delta}$, the Laplace mechanism with parameter $b$ is $\tuple{b \cdot \func{\ln}{\frac{1}{\delta}}}{\delta}$-Sample Accurate.
\end{theorem}

\begin{definition} [Gaussian Mechanism]
Given $0 \le \sigma$ and a query $\query \in \queriesFamily_{\Delta}$, the Gaussian mechanism with parameter $\sigma$ is defined as:
\[
\mechanismFunc{\sampleSet, \query} = \func{\query}{\sampleSet} + \text{G}_{\sigma}
\]
where $\text{G}_{\sigma}$ is a random variable with unbiased Gaussian distribution and standard deviation $\sigma$.
\end{definition}

\begin{theorem}[Gaussian Mechanism is Differentially Private]
Given $0 \le \sigma, \epsilon$, $0 < \delta \le 1$, and a query $\query \in \queriesFamily_{\Delta}$, the Gaussian mechanism with parameter $\sigma$ is $\tuple{\frac{2 \Delta \sqrt{2 \func{\ln}{1.25 / \delta}}}{\sampleSize \sigma}}{\delta}$-DP.
\end{theorem}

\begin{theorem}[Gaussian Mechanism is Sample Accurate]
Given $0 \le \sigma, \epsilon$, $0 < \delta \le 1$, and a query $\query \in \queriesFamily_{\Delta}$, the Gaussian mechanism with parameter $\sigma$ is $\tuple{\frac{\epsilon}{\sqrt{2 \func{\ln}{\sqrt{2} / \pi \delta}}}}{\delta}$-Sample Accurate.
\end{theorem}
\section{Distance Measures on Distributions} \label{apd:DistMeasures}

These distance measures between distributions will be used in various places in the paper.

\begin{definition}[Statistical Distance]\label{def:statDist}
The \emph{Statistical Distance} (also know as \emph{Total Variation Distance}) between two probability distributions $\distInd{1}, \distInd{2}$ over some domain $\range$ is defined as,

\begin{align*}
\func{\text{SD}}{\distInd{1}, \distInd{2}} & \coloneqq \underset{\rangeSet \in \range}{\max} \left( \func{\distInd{1}}{\rangeSet} - \func{\distInd{2}}{\rangeSet} \right)
\\ & = \underset{\rangeSet \in \range}{\max} \left( \func{\distInd{2}}{\rangeSet} - \func{\distInd{1}}{\rangeSet} \right)
\\ & = 
\frac{1}{2} \cdot \sum_{\rangeVal \in \range} \left| \func{\distInd{1}}{\rangeVal} - \func{\distInd{2}}{\rangeVal} \right|.
\end{align*}

The maximal set in the first definition is simply the set of all $\rangeVal$'s for which $\func{\distInd{1}}{\rangeVal} > \func{\distInd{2}}{\rangeVal}$ and for the second - the set of all $\rangeVal$'s for which $\func{\distInd{1}}{\rangeVal} < \func{\distInd{2}}{\rangeVal}$
\end{definition}

\begin{definition}[$\delta$-approximate max divergence]
The \emph{$\delta$-approximate max divergence} between two probability distributions $\distInd{1}, \distInd{2}$ over some domain $\range$ is defined as
\[
\func{\mathbf{D}_{\infty}^{\delta}}{\distInd{1} \Vert \distInd{2}} \coloneqq \max_{\rangeSet \subseteq \text{Supp}\left(\distInd{1}\right) \wedge \func{\distInd{1}}{\rangeSet} \ge \delta} \func{\ln}{\frac{\func{\distInd{1}}{\rangeSet} -\delta} {\func{\distInd{2}}{\rangeSet}}}.
\]
The case where $\delta=0$ is simply called the \emph{max divergence}.
\end{definition}

\begin{definition}[Indistinguishable distributions]\label{def:indist}
Two probability distributions $\distInd{1}, \distInd{2}$ over some domain $\range$ will be called \emph{$\tuple{\epsilon}{\delta}$-indistinguishable} if
\[
\max\left\{\func{\mathbf{D}_{\infty}^{\delta}}{\distInd{1} \Vert \distInd{2}}, \func{\mathbf{D}_{\infty}^{\delta}}{\distInd{2} \Vert \distInd{1}}\right\} \le \epsilon.
\]
this can also be written as the condition that for any $\rangeSet \subseteq \range$
\[
\func{\distInd{1}}{\rangeSet} \le e^{\epsilon} \cdot \func{\distInd{2}}{\rangeSet} + \delta \,\,\,\, \text{and} \,\,\,\,
\func{\distInd{2}}{\rangeSet} \le e^{\epsilon} \cdot \func{\distInd{1}}{\rangeSet} + \delta 
\]
\end{definition}

\begin{definition}[Maximal Leakage, based on \cite{IWK18}] \label{dfn:maxLkg}
Given two finite domains $\mathcal{X}, \mathcal{Y}$ and a joint distribution $\distJoint{\mathcal{X}}{\mathcal{Y}}{}$ defined over $\mathcal{X} \times \mathcal{Y}$,
The \emph{Maximal Leakage} between two marginal distributions $\distInd{\mathcal{X}}, \distInd{\mathcal{Y}}$ is defined as,
\[
\func{\mathcal{L}}{\distInd{\mathcal{X}} \rightarrow \distInd{\mathcal{Y}}} \coloneqq \func{\log}{\sum_{y \in \mathcal{Y}} \underset{x \in \mathcal{X} \,|\, \distFunc{x} > 0}{\max}\distFuncDep{y}{x}}.
\]
\end{definition}
\end{document}